\documentclass[11pt]{article}

\usepackage[round]{natbib}

\usepackage{url}            
\usepackage{booktabs}       
\usepackage{amsfonts}       
\usepackage{nicefrac}       
\usepackage{microtype}      
\usepackage{xcolor}         

\usepackage{amsmath,mathtools}
\usepackage{enumitem} 
\usepackage{amsthm}         
\usepackage{caption}        
\usepackage{subcaption}     
\usepackage[linesnumbered,ruled,vlined]{algorithm2e} 
\SetCommentSty{mycommfont}
\SetKwInput{KwInput}{Input}                
\SetKwInput{KwOutput}{Output}              
\usepackage{xr-hyper}                   
\usepackage{hyperref}       
\usepackage{comment}        
\usepackage{multirow}
\usepackage[createShortEnv]{proof-at-the-end}   

\newtheorem{thm}{Theorem}
\newtheorem{cor}[thm]{Corollary}
\newtheorem{lem}[thm]{Lemma}

\newtheorem{defn}{Definition}

\def\Cov{\mbox{Cov}}

\def\tr{\mbox{trace}}

\def\diag{\mbox{diag}}

\def\diag{\mbox{Diag}}

\def\uv{\mathbf u}
\def\vv{\mathbf v}
\def\wv{\mathbf w}
\def\xv{\mathbf x}
\def\yv{\mathbf y}
\def\zv{\mathbf z}

\def\Av{\mathbf A}

\def\Iv{\mathbf I}

\def\Ov{\mathbf O}
\def\Pv{\mathbf P}

\def\Rv{\mathbf R}

\def\Vv{\mathbf V}
\def\Wv{\mathbf W}
\def\Xv{\mathbf X}

\def\Zv{\mathbf Z}

\newcommand{\Sigmav}{\mbox{\boldmath{$\Sigma$}}}
\newcommand{\Lambdav}{\mbox{\boldmath{$\Lambda$}}}

\newcommand{\Dc}{\mathcal{D}}

\newcommand{\Fc}{\mathcal{F}}

\newcommand{\Ic}{\mathcal{I}}

\newcommand{\Uc}{\mathcal{U}}
\newcommand{\Vc}{\mathcal{V}}
\newcommand{\Wc}{\mathcal{W}}
\newcommand{\Xc}{\mathcal{X}}

\newcommand{\Rb}{\mathbb{R}}

\newcommand{\gr}{{\rm{Gr}}}

\newcommand{\bp}{{\rm{bp}}}
\newcommand{\out}{{\rm{out}}}

\def\diag{\mbox{diag}}

\def\tr{\mbox{Tr}}
\def\1v{\mathbf 1}
\def\0v{\mathbf 0}

\title{Subspace Recovery in Winsorized PCA: Insights into Accuracy and Robustness}
\author{Sangil Han, Kyoowon Kim, and  Sungkyu Jung \\
Department of Statistics Seoul National University}
\begin{document}
\maketitle

\begin{abstract}
In this paper, we explore the theoretical properties of subspace recovery using Winsorized Principal Component Analysis (WPCA), utilizing a common data transformation technique that caps extreme values to mitigate the impact of outliers. Despite the widespread use of winsorization in various tasks of multivariate analysis, its theoretical properties, particularly for subspace recovery, have received limited attention. We provide a detailed analysis of the accuracy of WPCA, showing that increasing the number of samples while decreasing the proportion of outliers guarantees the consistency of the sample subspaces from WPCA with respect to the true population subspace. Furthermore, we establish perturbation bounds that ensure the WPCA subspace obtained from contaminated data remains close to the subspace recovered from pure data. Additionally, we extend the classical notion of breakdown points to subspace-valued statistics and derive lower bounds for the breakdown points of WPCA. Our analysis demonstrates that WPCA exhibits strong robustness to outliers while maintaining consistency under mild assumptions. A toy example is provided to numerically illustrate the behavior of the upper bounds for perturbation bounds and breakdown points, emphasizing winsorization's utility in subspace recovery.
\end{abstract}

\section{INTRODUCTION}

Winsorization, often referred to as ``clipping,'' has long been recognized as a common and effective tool for handling extreme values in data analysis. Winsorization involves capping extreme values by projecting data points lying outside a specified boundary onto that boundary, ensuring that the support of the transformed data becomes bounded within a specified radius. This transformation is widely used in various fields, including differential privacy, where bounded data support is required to set scales for privacy-preserving noise \citep{karwaFiniteSampleDifferentially2017a, abadiDeepLearningDifferential2016b, kamathDifferentiallyPrivateAlgorithms2019a, biswasCoinPressPracticalPrivate2020a}. Winsorization also effectively handles outliers, particularly from heavy-tailed distributions or corrupted data, by reducing their influence on subsequent analyses without excluding data points \citep{bickelRobustEstimatesLocation1965a,yaleWinsorizedRegression1976a}. This ability to mitigate the impact of anomalies while preserving the overall dataset makes winsorization a frequently adopted technique in multivariate analysis, robust statistics, and other applications \citep{joseSimpleRobustAverages2008a,beaumontChapter11Dealing2009a}.

In the context of high-dimensional data, where dimension reduction and subspace recovery are crucial, winsorization has been incorporated as a preprocessing step to enhance robustness against anomalies and outliers. Dimension reduction is essential for summarizing high-dimensional datasets by identifying a lower-dimensional subspace that retains the significant variance of the data. Principal Component Analysis (PCA) is the most commonly used method for subspace recovery, but its sensitivity to outliers has led researchers to explore robust alternatives. Various approaches, including optimization-based methods, robust covariance estimation, and subsampling techniques, often involve complex optimization or filtering processes with heavy computational burdens \citep{candesRobustPrincipalComponent2011a,brahmaReinforcedRobustPrincipal2018a,zhangRobustRegularizedSingular2013a}. In contrast, data transformations such as winsorization offer a convenient and scalable solution. Winsorization can be applied universally before performing analyses, ensuring that subsequent analyses operate on transformed data with reduced outlier influence. This versatility has made winsorization a valuable tool in a wide range of high-dimensional applications, from dimension reduction to other forms of multivariate analysis.

Despite the widespread use of winsorization in practice, the theoretical foundations of its impact on subspace recovery have not been fully established. While empirical results have demonstrated its effectiveness in controlling the influence of outliers, there remains a significant gap in understanding how winsorization affects the accuracy and robustness of PCA from a theoretical standpoint. While \cite{raymaekersGeneralizedSpatialSign2019a,leyderGeneralizedSphericalPrincipal2024a} demonstrated the robustness of the covariance matrix of a winsorized random vector in terms of the influence function of eigenvectors and the breakdown of eigenvalues, {their results do not address the case where the number of variables $p$ increases, nor do they explore how the winsorization radius (clipping threshold) interacts with $p$ in the high-dimensional model.} Furthermore, the effects of winsorization on subspace recovery, particularly in terms of consistency and breakdown points, have yet to be rigorously quantified.


We contribute to the theoretical understanding and robustness of Winsorized PCA (WPCA) in subspace recovery, offering new insights into its consistency and breakdown points.

\textbf{Accuracy and Consistency in Subspace Recovery.}
We derive concentration bounds (in Theorem~\ref{thm:tail_ang}) for the PC subspace obtained through WPCA under a broad class of elliptical distributions \citep{cambanisTheoryEllipticallyContoured1981a,kelkerDistributionTheorySpherical1970a,kolloAdvancedMultivariateStatistics2006a}, which generalize multivariate Gaussian distributions and account for both heavy and light-tailed behavior. The derived concentration bounds for the principal angles between the sample WPCA subspace and the population subspace demonstrate that the sample subspace converges as the sample size increases and the proportion of contamination decreases. Additionally, we demonstrate that WPCA maintains consistency even with extremely large winsorization radius in the subgaussian case, where the distribution has light tails. We further validate the performance of our concentration bounds through a simulation study in high-dimensional settings. The results show that while the concentration bounds perform well in practice, they are not fully optimized, suggesting potential for further improvement.

\textbf{Strong and Weak breakdown.}
We introduce a new notion of strong breakdown (Definitions~\ref{def:strong_bp} and \ref{def:strongbp_subspace}), which offers a more sensitive measure of breakdown compared to the traditional notion. In subspace recovery, while traditional breakdown implies partial orthogonality between corrupted and uncorrupted subspaces \citep{hanRobustSVDMade2024b}, strong breakdown implies full orthogonality. This provides a more refined understanding of estimator behavior in extreme scenarios. We apply both strong and weak breakdown concepts to WPCA, providing a detailed analysis of its robustness.

\textbf{Breakdown Point Analysis for WPCA and traditional PCA.}
We show in Theorem~\ref{thm:breakdown} that the (strong) breakdown point of the $d$-dimensional subspace from WPCA has a lower bound proportional to the ratio of the (averaged) eigenvalue gap of the sample covariance of the winsorized data to the square of the winsorization radius, indicating WPCA’s resistance to contamination. In contrast, the breakdown points for traditional PCA are much smaller than those of WPCA. This demonstrates WPCA's superior robustness in subspace recovery. 
We confirm, in a simulated data example, our lower bound is indeed effective. 

\textbf{Robustness through Perturbation Bounds.}
We demonstrate that the PC subspaces obtained through WPCA not only resist breakdown under contamination but also experience minor perturbation when comparing subspaces from uncontaminated and contaminated data (Theorem~\ref{thm:perturbation}). We derive perturbation bounds for WPCA, showing that the deviation in the recovered subspace scales linearly with the level of contamination. These bounds confirm WPCA's robustness, indicating that it can tolerate small amounts of corruption without significant deviation in the subspace recovery.



\section{WINSORIZED PCA
}\label{sec:WPCA}
We implement WPCA as follows. Let $\Xv = [\xv_1, \dots, \xv_n]' \in \Rb^{n \times p}$ represent a centered, potentially contaminated data matrix consisting of $n$ samples with $p$ variables. 
The winsorized dataset is denoted by $\Xv^{(r)} = [\xv_{1}^{(r)}, \dots, \xv_{n}^{(r)}]'$, where each winsorized observation is defined as: 
\begin{align}\label{eq:winsor_sample}
\xv_{i}^{(r)} := \begin{cases}
    \xv_{i} & \text{ if }\|\xv_{i}\|_2 \leq r,\\
    \frac{r \xv_{i}}{\|\xv_{i}\|_2} & \text{ if }\|\xv_{i}\|_2 > r,
\end{cases}
\end{align} where $r > 0$ is the winsorization radius. The winsorization radius $r$ defines the boundary beyond which data points are projected onto the surface of a radius-$r$ ball.


Let $\Vc_d^{(r)}(\Xv)$ denote the $d$-dimensional PC subspace spanned by the eigenvectors corresponding to the largest $d$ eigenvalues of the winsorized sample covariance matrix, $\frac{1}{n}(\Xv^{(r)})'(\Xv^{(r)})$. We call this subspace $d$-dimensional winsorized (sample) PC subspace.

Infinitesimally small radius $r$ corresponds to a limiting case of winsorization, where all observations are normalized, which is equivalent to the transformation used in Spherical PCA (SPCA) by \citep{locantoreRobustPrincipalComponent1999a}, and other methods using normalization \citep{mardenRobustEstimatesPrincipal1999a,visuriSubspacebasedDirectionofarrivalEstimation2001a, taskinenRobustifyingPrincipalComponent2012a, hanRobustSVDMade2024b}. On the other hand, when the radius $r$ is sufficiently large such that $r \geq \max_i\{\|\xv_{i,\epsilon}\|_2\}$, no data points are winsorized, and WPCA coincides with traditional PCA. 

{When performing WPCA on a given dataset, any efficient Singular Value Decomposition (SVD) algorithm can be applied to the winsorized data \eqref{eq:winsor_sample}. Numerous studies and implementations of SVD have been developed to handle high-dimensional or large-sample datasets effectively. Factors such as the gap between singular values, the number of rows or columns, and the sparsity (the number of nonzero elements) of the data matrix can influence the choice of the SVD algorithm. Under various scenarios, fast and accurate SVD implementations, such as algorithms proposed by \cite{allen-zhuLazySVDEvenFaster2016a,muscoRandomizedBlockKrylov2015a,bhojanapalliGlobalOptimalityLocal2016a} can be utilized.}

Our analysis focuses on how closely this estimated subspace $\Vc_d^{(r)}(\Xv)$ approximates the (population) true subspace (in Section~\ref{sec:stat_accuracy}) and the target subspace derived from the uncontaminated dataset $\Xv_0$ (in Section~\ref{sec:robustness}). 

\section{
ACCURACY OF WINSORIZED PCA}\label{sec:stat_accuracy}
A zero-mean random vector $\xv \in \Rb^p$ is said to follow an \emph{elliptical distribution} with covariance matrix $\Sigmav$, if, 
for any orthogonal matrix $\Rv \in \Rb^{p \times p}$,
\begin{align}\label{eq:elliptical}
\Rv \Sigmav^{-\frac{1}{2}} \xv \overset{\rm{d}}{=} \Sigmav^{-\frac{1}{2}} \xv,
\end{align}
where $\xv\overset{\rm{d}}{=}\yv$ means $\xv$ and $\yv$ have the same distribution. Elliptical distributions, which include multivariate normal and $t$-distributions, generalize multivariate normal distributions by preserving elliptical symmetry \eqref{eq:elliptical}. Elliptical distributions characterized by a covariance matrix $\Sigmav$ form a family of distributions defined solely by their elliptical symmetry. A distribution belonging to the family of elliptical distributions cannot be fully specified by its covariance matrix alone, as it may exhibit either heavy or light tails. We use the notation $\xv \sim \Fc_{\Sigmav}$ to indicate that $\xv$ follows an elliptical distribution with covariance matrix $\Sigmav$, allowing us to encompass a wide range of random vectors with various tail behaviors. One notable property is that for any $\xv \sim \Fc_{\Sigmav}$ with population covariance matrix $\Sigmav$, winsorization of $\xv$ preserves the eigenvectors and the order of eigenvalues in the covariance matrix \citep{raymaekersGeneralizedSpatialSign2019a}, allowing us to infer the eigenstructure of the population covariance matrix even after winsorizing the data points.

To model contamination in the data, we introduce a contamination parameter $\epsilon \in [0, 0.5)$, representing the proportion of corrupted data points among $n$ data points. 
We assume that the uncontaminated $(1-\epsilon)n$ data points in $\Xv$, are i.i.d. realizations of a random vector $\xv \sim \Fc_{\Sigmav}$, and the contaminated $\epsilon n$ data points in $\Xv$ follow an arbitrary distribution. We denote the $\epsilon$-contaminated dataset by $\Xv = \Xv_{\epsilon} = [\xv_{1,\epsilon}, \dots, \xv_{n,\epsilon}]'$, and the set of indices corresponding to the contaminated data points by $\Ic_{\epsilon}$,
with $|\Ic_{\epsilon}| = \epsilon n$. When $\epsilon = 0$, the contaminated dataset becomes the uncontaminated dataset $\Xv_0$ with $n$ realizations of $\xv\sim\Fc_{\Sigmav}$.

Note that in Section~\ref{sec:robustness}, we will remove the distributional assumption. In this case, $\epsilon$-contamination will be allowed to occur at arbitrary positions in the pure dataset $\Xv_0$ with arbitrary values.

\subsection{PC Subspace Concentration}\label{sec:concentration}
In this section, we provide concentration inequalities for the winsorized PC subspace $\Vc_d^{(r)}(\Xv_\epsilon)$. This subspace is obtained by applying the traditional PCA on the winsorized contaminated dataset $\Xv_{\epsilon}^{(r)}$ as described in Section~\ref{sec:WPCA}. We demonstrate how the 
subspace 
$\Vc_d^{(r)}(\Xv_\epsilon)$ concentrates around the target subspace. Let the population covariance matrix have eigendecomposition: $\Cov(\xv) = \Sigmav = \Vv\Lambdav\Vv'$ where $\Vv'\Vv = \Iv_p$, and $\Lambdav = \diag(\lambda_1, \dots, \lambda_p)$ with $\lambda_1 \geq \dots \geq \lambda_p > 0$. Here $\Vv = [\vv_1,\dots,\vv_p]$ contains the eigenvectors of $\Sigmav$ and $\lambda_j$ are the corresponding eigenvalues. Our target population PC subspace is the $d$-dimensional subspace spanned by the first $d$ eigenvectors: 
\[
    \Vc_d = {\rm{span}}(\vv_1,\dots,\vv_d).
\]
We adopt the largest principal angle as a metric to measure the difference between subspaces \citep{wedinAnglesSubspacesFinite1983a,knyazevMajorizationChangesAngles2007a,qiuUnitarilyInvariantMetrics2005a}. For two given $d$-dimensional subspaces $\Uc$ and $\Wc$ of $\Rb^p$, the $d$ principal angles $0\leq \theta_1 \leq \dots \leq \theta_d$ between $\Uc$ and $\Wc$ are defined as follows. The smallest principal angle $\theta_1 $ between $\Uc$ and $\Wc$ is 
\begin{align}\label{eq:min_principal_angle}
    \cos(\theta_1) = \max_{\uv \in \Uc}\max_{\wv \in \Wc} \frac{|\uv'\wv|}{\|\uv\|_2\|\wv\|_2} = \frac{|\uv_1'\wv_1|}{\|\uv_1\|_2\|\wv_1\|_2}
\end{align}
where $\uv_1\in \Uc$ and $\wv_1 \in \Wc$ are the vectors satisfying $\cos(\theta_1) = \frac{|\uv_1'\wv_1|}{\|\uv_1\|_2\|\wv_1\|_2}$. The subsequent principal angles $\theta_j$ $(j=1,\dots,d)$ are defined recursively by:
\begin{align}\label{eq:principal_angle}
\cos(\theta_j) = \max_{\uv \in \Uc}\max_{\wv \in \Wc} \frac{|\uv'\wv|}{\|\uv\|_2\|\wv\|_2} = \frac{|\uv_j'\wv_j|}{\|\uv_j\|_2\|\wv_j\|_2}
\end{align}
subject to $\uv'\uv_k = 0$ and $\wv'\wv_k=0$ for $k=1,\dots,j-1$. The largest principal angle $\theta_d$ provides an upper bound on the deviation between the subspaces. Since $\theta_1 \leq \dots \leq \theta_d$, if $\theta_d = 0$, then all principal angles are zero, which implies that the two subspaces $\Uc$ and $\Wc$ coincide. In this context, we denote $\Theta(\Uc,\Wc) = \theta_d$ for the largest principal angle.

Let $\Theta_{\epsilon}^{(r)} = \Theta(\Vc_d^{(r)}(\Xv_{\epsilon}), \Vc_d)$ be the largest principal angle between $\Vc_d^{(r)}(\Xv_{\epsilon})$ and $\Vc_d$. 
We present the following theorem to establish the consistency of the winsorized PC subspace.
\begin{thm}\label{thm:tail_ang}
Assume $\xv_i|_{i\not\in\Ic_{\epsilon}}, \overset{\rm{i.i.d}}{\sim} \Fc_{\Sigmav}$ follow an elliptical distribution and $\lambda_d > \lambda_{d+1}$. Let $\lambda_j^{(r)}$ denote the $j$th largest eigenvalue of $\Cov(\xv^{(r)})$, where $\xv^{(r)}$ is the winsorized random vector of $\xv \sim \Fc_{\Sigmav}$. {For any $n$ and $p$, }
    \begin{align}\label{eq:tail-nongaussian}
    \begin{split}
        E[\sin \Theta_{\epsilon}^{(r)}] 
        \leq& \frac{2r^2\epsilon}{\lambda_d^{(r)} -\lambda_{d+1}^{(r)}}+\frac{2^8(\frac{r^2\lambda_1}{p\lambda_p}) (\sqrt{\frac{p}{n}} \vee \frac{p}{n})}{\lambda_d^{(r)} -\lambda_{d+1}^{(r)}}.
    \end{split}
    \end{align}
    Moreover, if 
    \begin{align}\label{eq:subgaussian}
        \sup_{\vv \in S^{p-1}} E[(\vv'\Sigmav^{-\frac{1}{2}}\xv)^{2k}] \leq \frac{(2k)!}{2^k k!} \sigma^{2k}
    \end{align}
     for all $k = 1,2,\dots$ with some $\sigma > 0$, then
    \begin{align}\label{eq:tail-subgaussian}
    \begin{split}
        E[\sin \Theta_{\epsilon}^{(r)}] 
        \leq& \frac{2r^2\epsilon}{\lambda_d^{(r)} -\lambda_{d+1}^{(r)}}+\frac{2^8\lambda_1(\frac{r^2}{p\lambda_p} \wedge \sigma^2) (\sqrt{\frac{p}{n}} \vee \frac{p}{n})}{\lambda_d^{(r)} -\lambda_{d+1}^{(r)}}.
    \end{split}
    \end{align}
\end{thm}
The assumption $\eqref{eq:subgaussian}$ states that $\yv := \Sigmav^{-\frac{1}{2}}\xv$ is $\sigma$-subgaussian \citep{wainwrightHighDimensionalStatisticsNonAsymptotic2019b, vershyninHighDimensionalProbabilityIntroduction2018a}. This subgaussian assumption implies that each component of the random vector $\yv$ exhibits tail behavior similar to that of a Gaussian distribution, meaning its tails decay exponentially.

Note that the winsorized eigenvalues $\lambda_j^{(r)}$ depend on the winsorization radius $r$, the eigenvalues of $\Sigmav$, and the number of variables $p$. To analyze the consistency of the winsorized PC subspace—in terms of convergence in mean—we consider how the parameters $r$, $p$, and $n$ interact.

\subsubsection{Effect of Winsorization Radius \texorpdfstring{$r$}{r}}\label{sec:effectR}
We begin with a remark on SPCA: One might conjecture that decreasing $r$, leading to SPCA, would cause the upper bound in \eqref{eq:tail-nongaussian} and \eqref{eq:tail-subgaussian} to converge to 0. However, since the ratio ${\lambda_j^{(r)}}/{r^2}$ converges to the $j$th eigenvalue of the covariance matrix of $\xv/\|\xv\|_2$, the upper bounds  do not vanish as $r$ decreases.

Fixing the number of variables $p$, we define $g(n, r) = \Omega(h(n, r))$ if there exist constants $a \leq b$ such that $a \leq g(n, r)/h(n, r) \leq b$. First, consider the case without outliers $(\epsilon = 0)$. As both the winsorization radius $r$ and the sample size $n$ increase, the upper bound becomes $\Omega\left(\frac{r^2}{\sqrt{n}}\right)$, since the eigenvalue gap $\lambda_d^{(r)} - \lambda_{d+1}^{(r)}$ converges to $\lambda_d - \lambda_{d+1}$ as $r$ grows. Consistency is guaranteed if $\frac{r^2}{\sqrt{n}}$ converges to zero. Without any assumptions on tail behavior, however, increasing $r^2$ too rapidly relative to $n$ may negatively impact estimation due to potential extreme values from heavy tails. In the presence of outliers $(\epsilon > 0)$, the deviation term ${2r^2\epsilon}/{(\lambda_d^{(r)} - \lambda_{d+1}^{(r)}) }$ grows at the rate $\Omega(r^2)$ resulting in an upper bound of $\Omega(r^2\epsilon + r^2/\sqrt{n})$ in \eqref{eq:tail-nongaussian}.


For light-tailed distributions and no outliers, a large $r$ with increasing $n$ ensures consistency, as WPCA approaches traditional PCA, which reliably captures PC directions as $n$ grows. The upper bound \eqref{eq:tail-subgaussian} reflects this scenario. When $\epsilon=0$, and $\xv$ is subgaussian, the upper bound becomes $\Omega(\frac{1}{\sqrt{n}})$ as $n$ and $r$ increase. In the presence of outliers, the upper bound simplifies to $\Omega(r^2\epsilon + 1/\sqrt{n}) = \Omega(r^2)$.


\begin{figure}[t]
\centering
  \begin{subfigure}[b]{0.4\textwidth}
    \centering
    \includegraphics[width=\linewidth]{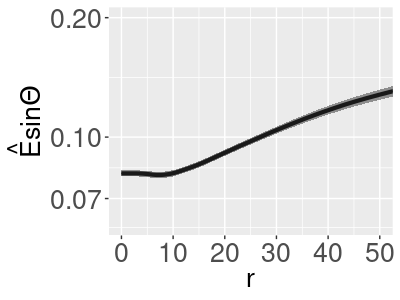}
    \caption{$\xv \sim t_3$, $\epsilon=0$} \label{subfig:sin00t}
  \end{subfigure}
  \begin{subfigure}[b]{0.4\textwidth}
    \includegraphics[width=\linewidth]{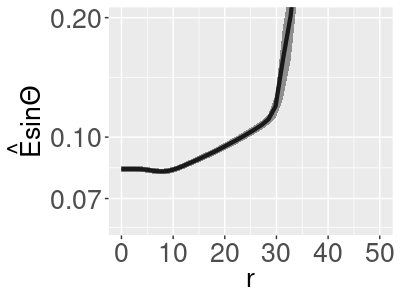}
    \caption{$\xv \sim t_3$, $\epsilon=0.05$} \label{subfig:sin05t}
  \end{subfigure}\\
  \begin{subfigure}[b]{0.4\textwidth}
    \centering
    \includegraphics[width=\linewidth]{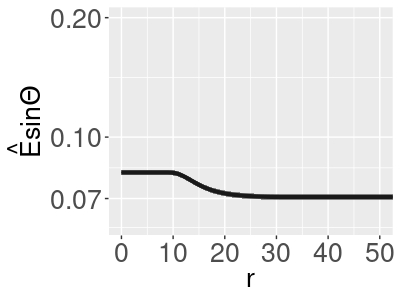}
    \caption{$\xv \sim N$, $\epsilon=0$} \label{subfig:sin00n}
  \end{subfigure}
  \begin{subfigure}[b]{0.4\textwidth}
    \centering 
    \includegraphics[width=\linewidth]{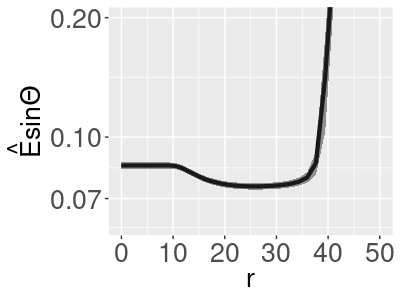}
    \caption{$\xv \sim N$, $\epsilon=0.05$} \label{subfig:sin05n}
  \end{subfigure}
\caption{Empirical expectation $\widehat{E}[\sin \Theta_{\epsilon}^{(r)}]$ for different tail behavior and contamination levels. Panels (a) and (b) show the results when   
$\xv_i$ follows a multivariate $t_{3}$-distribution, while (c) and (d)
represent the case where $\xv_i$ follows a multivariate Gaussian distribution. In each figure, $\epsilon$ denotes the proportion of contaminated data. }\label{fig:effectR}
\end{figure}

Figure~\ref{fig:effectR} illustrates the effect of the winsorization radius $r$ on the empirical expectation $\widehat{E}[\sin \Theta_{\epsilon}^{(r)}]$ (or, simply the `loss') in both heavy-tailed ($t_3$) and light-tailed (Gaussian) distributions. The data generation details are provided in the supplementary material. The upper figures correspond to the heavy-tailed $t_3$-distribution. In Figure~\ref{subfig:sin00t}, even in the absence of outliers, we observe that the loss increases, reaching approximately 0.19 as $r$ grows. When outliers are present, as shown in Figure~\ref{subfig:sin05t}, the loss rises significantly and approaches 1 as $r$ increases. The results suggest that the radius does not need to be infinitesimally small; there exists a non-zero radius $r$ where the loss is minimized in both cases.

In contrast, the lower figures depict the results for a multivariate Gaussian distribution, which has light tails. As shown in Figure~\ref{subfig:sin00t} and \ref{subfig:sin05t}, the behavior differs from that of the $t_3$ distribution. When there are no outliers, increasing $r$ slightly improves the loss. However, when outliers are present, as shown in Figure~\ref{subfig:sin05n}, the loss decreases slightly for small $r$, then increases sharply as $r$ continues to grow.

\subsubsection{Effect of Winsorization Radius \texorpdfstring{$r$}{r} in High Dimension}\label{sec:effectR_highp}
In this section, we examine the high-dimensional setting where the number of variables $p$ increases. We assume that for $j \geq p_0$, the eigenvalues of $\Sigmav$ remain constant at $\lambda_j = \lambda$ for some $p_0 > d$. To analyze the scenario where $n$, $p$, and $r$ increase together, we assume $r = p^{1/2 + \beta}$, where $\beta \in (-\infty, \infty)$. When $\beta = 0$, the radius is proportional to $\sqrt{p}$. Since the expected norm of the random vector is $E[\xv'\xv] = \sum_{j=1}^p \lambda_j = \Omega(p)$, setting $r = \sqrt{p}$ results in many data points being projected (by the winsorization), while a sufficient number remain un-projected. Positive $\beta$ implies fewer projected points, while negative $\beta$ means more projected points.
\begin{cor}\label{thm:tail_ang_cor}
    Assume $\xv_i|_{i\not\in\Ic_{\epsilon}}, \overset{\rm{i.i.d}}{\sim} \Fc_{\Sigmav}$ follow an elliptical distribution and $\lambda_d > \lambda_{d+1}$. Let $r = p^{1/2 + \beta}$ with $\beta \in (-\infty,\infty)$, and $C_1,C_2,$ and $C_3$ be positive absolute constants.
    \begin{align}\label{eq:tail-nongaussian_asy}
    \begin{split}
        E[\sin \Theta_{\epsilon}^{(r)}] 
        \leq& C_1 p^{1+2(\beta \vee 0)}\epsilon + C_2 p^{2(\beta \vee 0 )}(\sqrt{\frac{p}{n}} \vee \frac{p}{n} ).
    \end{split}
    \end{align}
    Moreover, if $\Sigmav^{-1/2}\xv$ is $\sigma$-subgaussian, then
    \begin{align}\label{eq:tail-subgaussian_asy}
    \begin{split}
        E[\sin \Theta_{\epsilon}^{(r)}] 
        \leq& C_1 p^{1+2(\beta \vee 0)}\epsilon + C_3(\sqrt{\frac{p}{n}} \vee \frac{p}{n} ).
    \end{split}
    \end{align}
\end{cor}
In both upper bounds \eqref{eq:tail-nongaussian_asy} and \eqref{eq:tail-subgaussian_asy}, the first term, related to the contamination proportion $\epsilon$, is dominant, growing at the rate $p^{1 + 2(\beta \vee 0)}$. Therefore, an excessively large radius $r = p^{1/2 + \beta}$ with $\beta > 0$ may cause significant distortion in the winsorized PC subspaces. On the other hand, when there are no outliers $(\epsilon = 0)$, a large sample size with $p/n$ converging to 0 guarantees consistency, provided that $\Sigma^{-1/2}\xv$ is subgaussian or the data is heavily winsorized with $\beta \le 0$. We numerically demonstrate the consistency of winsorized PC subspaces in the scenario with increasing $p$ in Section~\ref{sec:numeric}.

{It is known that the estimation of PC subspace has the minimax rate of $\sqrt{p/n}$ \citep{duchiSubspaceRecoveryHeterogeneous2022a, caiOptimalEstimationRank2015a,zhangHeteroskedasticPCAAlgorithm2022a,caiOptimalDifferentiallyPrivate2024c}. Our asymptotic upper bound can be compared with this rate. For a careful comparison between our rate involving the contamination rate $\epsilon$ and the minimax rate, we will assume that the number of contaminated observations is fixed. This simplification gives the rates $(\sqrt{p/n} + p/n)p^{2(\beta \vee 0)}$ for elliptical distributions, and $(\sqrt{p/n} + p^{1+ 2(\beta \vee 0)}/n)$ under additional sub-Gaussian assumption, for our error bounds. When $n > p$ and if we choose $\beta \leq 0$, our rates above become $O(\sqrt{p/n})$. Thus, our method achieves the minimax rate of $\sqrt{p/n}$, demonstrating strong performance even in the presence of contamination. While our rate $O(p/n)$ is sub-optimal when compared to the minimax lower bound for $p > n$, WPCA maintains both robustness and accuracy, even in the challenging scenarios where outliers are heavily contaminated.}

\subsection{Numerical Study}\label{sec:numeric}

We simulate the concentration bounds of winsorized PC subspaces in a high-dimensional setting, where the number of variables $p$ increases. For $k = 1,\ldots, 4$, we set the dimension $p_k = 1000k$, and the sample size $n_k$ to satisfy 
{$p_k/n_k = 1/(2k)$}. This ensures $p/n$ converges to 0 as $p$ increases. We generate data from two distributions: a heavy-tailed multivariate $t_3$ and a light-tailed Gaussian. The target subspace dimension is $d=2$, and we set two outliers with magnitudes proportional to $np$, positioned orthogonally to the target subspace; that is, $\epsilon =2/n$.

We test three winsorization radii: $r_{1k} = 1$, $r_{2k} = p_k^{1/2}$, and $r_{3k} = (p_k \log p_k )^{1/2}$. With respect to the parameterization of the radius in Corollary \ref{thm:tail_ang_cor}, $r_{lk} = p_k^{1/2 + \beta_l}$ with $\beta_1 = -1/2$ and $\beta_2 = 0$. The third radius $r_{3k}$ grows slightly faster than $r_{2k}$. These choices correspond to the effects of small, moderate, and large winsorization radii in high-dimensional settings.

\begin{figure}[t]
\centering
\begin{subfigure}[b]{1\textwidth}
    \centering
    \includegraphics[width=0.5\linewidth]{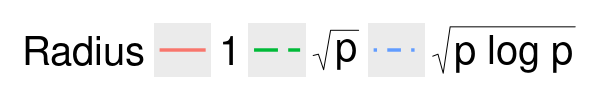}
  \end{subfigure}\\
  \begin{subfigure}[b]{0.4\textwidth}
    \centering
    \includegraphics[width=\linewidth]{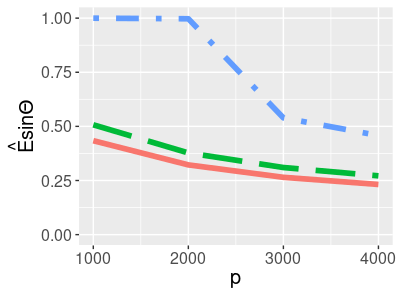}
    \caption{$\xv \sim t_3$} \label{subfig:numeric_t0}
  \end{subfigure}
  \begin{subfigure}[b]{0.4\textwidth}
    \includegraphics[width=\linewidth]{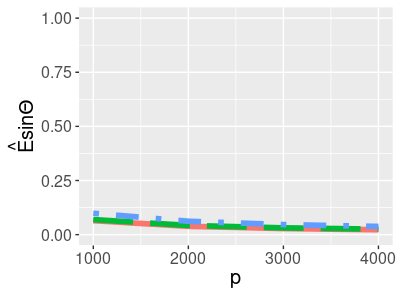}
    \caption{$\xv \sim t_3$, $\lambda_1,\lambda_2 \approx \sqrt{p} $} \label{subfig:numeric_t_half}
  \end{subfigure}\\
  \begin{subfigure}[b]{0.4\textwidth}
    \centering
    \includegraphics[width=\linewidth]{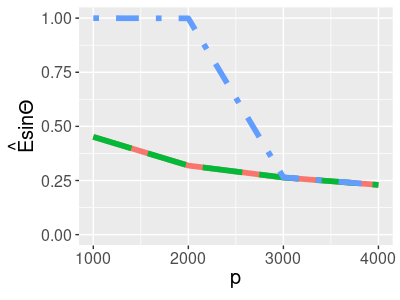}
    \caption{$\xv \sim N$} \label{subfig:numeric_n0}
  \end{subfigure}
  \begin{subfigure}[b]{0.4\textwidth}
    \centering 
    \includegraphics[width=\linewidth]{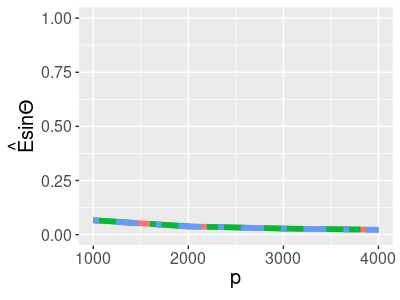}
    \caption{$\xv \sim N$, $\lambda_1,\lambda_2 \approx \sqrt{p} $} \label{subfig:numeric_n_half}
  \end{subfigure}
\caption{Empirical expectation $\widehat{E}[\sin \Theta_{\epsilon}^{(r)}]$ for different tail behaviors. Panels (a) and (c) show the results under non-spiked model with the $t_3$ and Gaussian distributions, respectively.    
Panels (b) and (d) represent the spiked model. }\label{fig:numeric_high_p}
\end{figure}

Figure~\ref{subfig:numeric_t0} displays the loss (empirical expectation $\widehat{E}[\sin \Theta_{\epsilon}^{(r)}]$) from the heavy-tailed $t_3$ distribution for which the eigenvalues of covariance matrix $\Sigmav$ are constant. 
As one can expect from \eqref{eq:tail-nongaussian_asy}, the losses 
decrease as $p$ grows. For this heavy-tailed distribution, smaller winsorization radius provides better accuracy. 


We next use a spiked covariance model for $\Sigmav$ where the first $d$ eigenvalues scale with $\sqrt{p}$. As shown in Figures~\ref{subfig:numeric_t_half}, the losses for all radii are smaller than those in the non-spiked model, and tend to zero. This is due to the higher signal-to-noise ratio inherent in the spiked model. 


The bottom panels of Figure~\ref{fig:numeric_high_p} correspond to the
Gaussian distribution. In these light-tailed cases,  the winsorization with moderate radius works as good as that with small radius. 

Note that when the larger radius ($r_{3k}$) is used, and for lower $p$, the outlier adversely affects the subspace estimates. 
For high $p$, since the magnitude of the outlier becomes larger than the winsorization radius, winsorization is effective, as can be inspected from Figures~\ref{fig:numeric_high_p}(a) and (c). 

\section{ROBUSTNESS OF WINSORIZED PCA 
}\label{sec:robustness}
Recall that $\Xv_0$ represents the uncontaminated data, and $\Xv_\epsilon$ is the contaminated dataset. 
In this section, we investigate two aspects of robustness: subspace breakdown points and perturbation bounds.

\subsection{Breakdown Point Analysis}
\subsubsection{Breakdown Points for Real-valued Statistics}
We focus on the concept of breakdown points \citep{hampelContributionsTheoryRobust1968a,bickelFestschriftErichLehmann1982a,huberFiniteSampleBreakdown1984a,huberRobustStatistics2011a}, which measure the robustness of a statistic against corrupted data. The breakdown point of a statistic is the minimum proportion of corrupted data required to make the statistic ``break down." For instance, the sample mean has a breakdown point of $1/n$, meaning a single outlier can drastically affect it, while the sample median, with a breakdown point of $1/2$, is more robust. Formally, the breakdown point of a real-valued statistic $f: \Xc^n \rightarrow \Rb$ at an $n$-sample $\Xv_0 \in \Xc^n$ is defined as
\begin{align}\label{eq:bp_scalar}
 \bp(f;\Xv_0) := \min_{1 \leq l \leq n}\{\frac{l}{n} : \sup_{\Zv_{l}}|f(\Zv_l) - f(\Xv_0)| = \infty\},
\end{align}
where the supremum is taken over the collection of all possible corrupted data $\Zv_l$, obtained by replacing $l$ data points in  $\Xv_0$ with arbitrary values.
A breakdown point $\bp(\Xv_0; f) = m/n$ represents a threshold of resistance, meaning that the statistic $f$ will not break down as long as the proportion of corruption does not exceed $m/n$. 

For the cases where $f(\Xv_0) \in \mathbb{R}^d$, many researchers have used a global dissimilarity measure in determining the breakdown of $f$ \citep{hubertHighBreakdownRobustMultivariate2008a,lopuhaaBreakdownPointsAffine1991b,beckerMaskingBreakdownPoint1999a,heRobustDirectionEstimation1992a}.
Typically, $|f(\Zv_l) - f(\Xv_0)|$ in \eqref{eq:bp_scalar} is replaced with $D(f(\Zv_l) - f(\Xv_0))$, where $D$ is the metric that quantifies the dissimilarity between vector-valued estimates. However, the breakdown of a multivariate estimator does not imply that all components break down simultaneously. 
As an instance, consider the five-number summary, $f(\Xv_0) = [Q_0, Q_1, Q_2, Q_3, Q_4]$, where $Q_0$ through $Q_4$ are the minimum, quartiles, and maximum. 
The breakdown point of $f$ is $1/n$ when using $D_{\Rb^5}(f(\Xv_0), f(\Zv_l)) = \|f(\Xv_0) - f(\Zv_l)\|_2$, because the minimum and maximum are sensitive to a single outlier. However, the median $Q_2$ has a breakdown point of $1/2$. 
To focus on $Q_2$ alone, we can use a modified dissimilarity function $\overline{D}_{\Rb^5}$ on $\Rb^5 \times \Rb^5$, defined as $\overline{D}_{\Rb^5}(f_1, f_2) = |f_{13} - f_{23}|$, where $f_i = (f_{i1}, \dots, f_{i5})'$,  
 with $f_{i3}$ representing the median component. By replacing $D_{\Rb^5}$ with $\overline{D}_{\Rb^5}$, the breakdown point increases to approximately $1/2$. 
 In the next section, we introduce a new notion of \textit{strong breakdown } to explain these different types of breakdown. 
 
 
 

\subsubsection{
Strong Breakdown
}\label{sec:strong_breakdown}
Consider a space $\Dc$ and a statistic $f : \Xc^n \rightarrow \Dc$. We measure the dissimilarity between $f(\Xv_0)$ and $f(\Zv_l)$ using a dissimilarity function $D: \Dc \times \Dc \rightarrow [0, \infty)$ satisfying $D(f, f) = 0$ for all $f \in \Dc$, and $\sup D > 0$. Note that the dissimilarity function $D$ may have two distinct elements $f_1 \neq f_2$ satisfying $D(f_1, f_2) = 0$. 
In the five-number summary example, $D_{\Rb^5}$ and $\overline{D}_{\Rb^5}$  are different dissimilarity functions on $\Rb^5$. 
The breakdown point of $f$ with respect to the dissimilarity function $D$ at $\Xv_0 \in \Xc^n$ is defined as
\begin{align}\label{eq:bp_multi}
 \bp(f,D;\Xv_0) := \min_{1 \leq l \leq n}\{\frac{l}{n} : \sup_{\Zv_{l}}D(f(\Zv_l), f(\Xv_0)) = \infty_D\}.
\end{align}
where $\infty_D := \sup_{f_1, f_2 \in \Dc} D(f_1, f_2)$ represents the maximal possible dissimilarity.


For two dissimilarity functions $D$ and $\overline{D}$ on $\Dc$, we say that $\overline{D}$ is weaker than $D$ (denoted $\overline{D} \preceq D$) if $\lim_{k \to \infty} D(f_{1k}, f_{2k}) = \infty_D$ for any two sequences ${f_{1k}}, {f_{2k}} \in \Dc$ satisfying $\lim_{k \to \infty} \overline{D}(f_{1k}, f_{2k}) = \infty_{\overline{D}}$. 
Simply put, if $\overline{D}(f,g) = \infty_{\overline{D}}$ gives  ${D}(f,g) = \infty_{{D}}$, then  $\overline{D} \preceq D$.
%
%
%
This relation implies that $\overline{D}$ is less sensitive to breakdown (reaching the maximal difference) than $D$. For any $\Xv_0 \in \Xc^n$ and statistic $f: \Xc^n \rightarrow \Dc$, if $\overline{D} \preceq D$, then:
\[
\bp(f,\overline{D};\Xv_0) \geq \bp(f,D;\Xv_0).
\]
This result indicates that using a weaker dissimilarity function leads to a higher (stronger) breakdown point. In practice, this means that an estimator may appear more robust when assessed with respect to a weaker dissimilarity function, focusing on specific components or aspects of the estimator.  Given two dissimilarity functions $\overline{D} \preceq D$, we define strong breakdown point as follows:
\begin{defn}\label{def:strong_bp}
    For two given breakdown points, $\bp(f, \overline{D}; \cdot)$ and $\bp(f, D; \cdot)$, with $\overline{D} \preceq D$, we say that $\bp(f, \overline{D}; \cdot)$ is the strong breakdown point and $\bp(f, D; \cdot)$ is the (weaker) breakdown point.
\end{defn}

\subsubsection{Breakdown Points for Subspace-valued statistics}
Our interest lies in PC subspaces, thus we extend the notion of breakdown points to subspace-valued statistics using the largest and the smallest principal angles, as dissimilarity functions, on the Grassmannian manifold $\gr(d,p)$, the set of all $d$-dimensional linear subspaces in $\Rb^p$. Let $\Vc : \Rb^{n \times p} \rightarrow \gr(d,p)$ be the subspace-valued statistic of interest. An example is the $d$-dimensional PC subspace derived from data $\Xv_0$. 
The smallest and the largest principal angles defined in \eqref{eq:min_principal_angle} and \eqref{eq:principal_angle} are denoted by $\theta(\cdot,\cdot)$ and $\Theta(\cdot,\cdot)$, respectively. 

\begin{defn}
\label{def:strongbp_subspace}
For $\Vc \in \gr(d,p)$ and $\Xv_0 \in \Rb^{n\times p}$, the breakdown point of $\Vc \in \gr(d,p)$ at $\Xv_0$ is 
    \begin{align}\label{eq:bp_space}
    \begin{split}
 \bp(\Vc;&\Xv_0) := \bp(\Vc, \Theta;\Xv_0)\\ 
 &= \min_{1 \leq l \leq n}\{\frac{l}{n} : \sup_{\Zv_{l}}\Theta(\Vc(\Zv_l),\Vc(\Xv_0)) = \frac{\pi}{2} \}
 \end{split}
 \end{align}
 and the strong breakdown point of $\Vc: \Rb^{n\times p } \rightarrow \gr(d,p)$ at $\Xv_0$ is
 \begin{align}\label{eq:strong_bp_space}
 \begin{split}
 \overline{\bp}(\Vc;&\Xv_0) := \bp(\Vc, \theta;\Xv_0)  \\
 &= \min_{1 \leq l \leq n}\{\frac{l}{n} : \sup_{\Zv_{l}}\theta(\Vc(\Zv_l),\Vc(\Xv_0)) = \frac{\pi}{2} \}.
 \end{split}
\end{align}
\end{defn}
The breakdown point \eqref{eq:bp_space} was proposed in \cite{hanRobustSVDMade2024b}. 
 For $d \leq p/2$, the strong breakdown point $\overline{\bp}(\Vc;\Xv_0)$ is always greater than or equal to the breakdown point $\bp(\Vc;\Xv_0)$, since $\theta \preceq \Theta$. When $d > p/2$, 
any two $d$-dimensional subspaces must intersect, and strong breakdown never occurs. (In fact, (\ref{eq:strong_bp_space}) is ill-defined for this case.)  
%
Strong breakdown implies that the subspace derived from contaminated data becomes fully orthogonal to the subspace obtained from uncontaminated data. In contrast, weak breakdown occurs when contaminated subspace is only partially orthogonal to its uncontaminated counterpart.

\subsection{Breakdown Points of Winsorized PCA}\label{sec:breakdown_winsor}
In this section, we begin by examining the lack of robustness of traditional PCA, clarifying the breakdown and strong breakdown points of $d$-dimensional PC subspaces. 

\begin{thm}\label{thm:breakdown_PCA}
Let $\Vc_d(\Xv_0)$ be given by the $d$-dimensional PC subspace obtained from traditional PCA applied to the data $\Xv_0$, and $\widehat{\lambda}_j$ be the $j$th largest eigenvalue of $\Xv_0'\Xv_0/n$. Assume that $\widehat{\lambda}_d > \widehat{\lambda}_{d+1}$. Then,
\[
 \bp(\Vc_d;\Xv_0) = \frac{1}{n},\text{ and } \overline{\bp}(\Vc_d ;\Xv_0)  = \frac{d}{n}.
\]
\end{thm}
This theorem highlights that traditional PCA is highly sensitive to outliers. A single outlier can significantly impact the estimation of the PC subspace, as reflected by the low breakdown point $\bp(\Vc_d ; \Xv_0) = 1/n$. The strong breakdown point $\overline{\bp}(\Vc_d ; \Xv_0)$ increases with the dimension $d$ of the subspace. However, since the dimension $d$ is typically much smaller than the number of samples $n$, even a small fraction of contaminated data can cause substantial distortion.

We provide lower bounds for the breakdown points of winsorized PC subspaces, indicating the robustness of WPCA compared to traditional PCA.
\begin{thm}\label{thm:breakdown}
Let $\Vc_d^{(r)}$ be a $d$-dimensional PC subspace from WPCA. Then,
\begin{align}\label{eq:bp_win}
    \begin{split}
        \bp(\Vc_d^{(r)} ;\Xv_0) 
    &\geq \frac{1}{2r^2} (\widehat{\lambda}_d^{(r)}- \widehat{\lambda}_{d+1}^{(r)}), \\
    \overline{\bp}(\Vc_d^{(r)} ;\Xv_0) 
    &\geq \sup_{d_0\leq d} \frac{\sum_{j=1}^{d_0}\widehat{\lambda}_j^{(r)} - \sum_{j=1}^{d_0}\widehat{\lambda}_{d+j}^{(r)}}{2r^2d_0}.
    \end{split}
\end{align}
\end{thm}

The lower bounds in \eqref{eq:bp_win} are less than or equal to $\frac{1}{2}$, as $\widehat{\lambda}_j^{(r)} \leq r^2$ for all $j=1,\dots,p$. 

%
%

\begin{figure}[t]
\centering
    \includegraphics[width=0.8\linewidth]{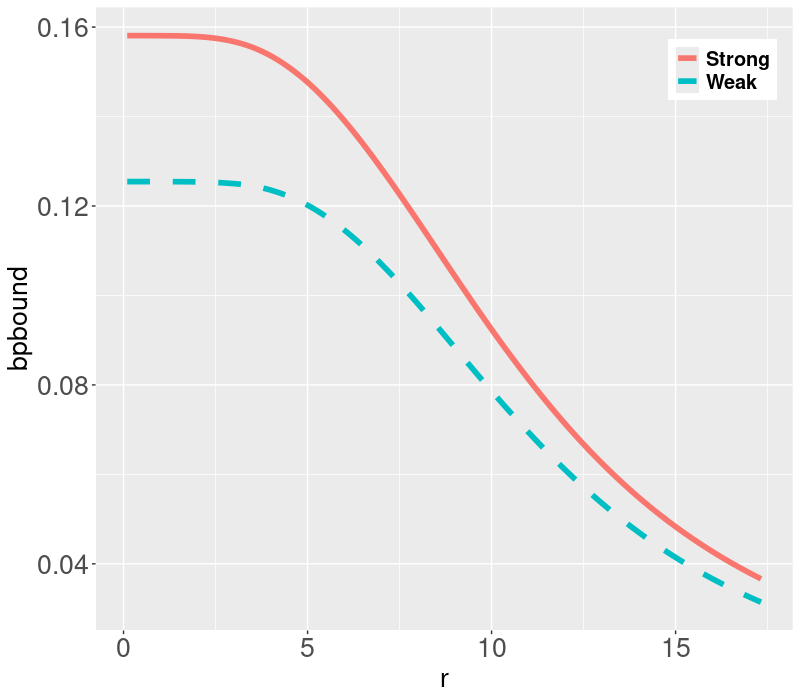}
    \caption{Estimated lower bounds for the breakdown points in \eqref{eq:bp_win}. 
    }\label{fig:lower_bdpt}
\end{figure}

We empirically observe that the smaller the radius $r$, the more robust the winsorized PC subspace becomes  in terms of both strong and weak breakdown points. 
%
%
Figure~\ref{fig:lower_bdpt} illustrates how the breakdown points vary as the winsorization radius $r$ varies. The lower bounds of each breakdown point decrease as $r$ increases. We observe that, for every $r$, the lower bound for the strong breakdown point is larger than that for the (weak) breakdown point. Moreover, the gap between the lower bounds becomes larger as $r$ decreases. All in all, WPCA appears to be less robust to contamination when using a larger winsorization radius, breaking down at lower contamination levels.

\subsection{Subspace 
Perturbation Bound
}\label{sec:pertur}
The notion of breakdown examines only the extreme cases in which the dissimilarity is maximized. Thus, there may be cases under which breakdown does not occur but the quality of the statistics from contaminated data is low.   
%
%
%
%
To inspect the amount of deviation of $\Vc_d(\Xv_\epsilon)$ from $\Vc_d(\Xv_0)$, 
we establish the subspace perturbation bound
using the largest principal angle.


\begin{thm}\label{thm:perturbation}
Let $\widehat{\lambda}_j^{(r)}$ be the $j$th largest eigenvalue of $\frac{1}{n}\Xv_0^{(r)\prime}\Xv_0^{(r)}$, and $\widehat{\Theta}_{\epsilon}^{(r)} = \Theta\left(\Vc_d^{(r)}(\Xv_{\epsilon}), \Vc_d^{(r)}(\Xv_0)\right)$ be the largest principal angle between $\Vc_d^{(r)}(\Xv_{\epsilon})$ and $\Vc_d^{(r)}(\Xv_0)$. 
If $\widehat{\lambda}_d^{(r)}- \widehat{\lambda}_{d+1}^{(r)} > 0$, then
\begin{align}\label{eq:upperbound1}
    \sin \widehat{\Theta}_{\epsilon}^{(r)} \leq \frac{2r^2 \epsilon}{\widehat{\lambda}_d^{(r)}- \widehat{\lambda}_{d+1}^{(r)}}.
    \end{align}
    Additionally, if $\widehat{\lambda}_d^{(r)}- \widehat{\lambda}_{d+1}^{(r)} > 4r^2\epsilon$, then
    \begin{align}\label{eq:upperbound2}
        \sin \widehat{\Theta}_{\epsilon}^{(r)} &\leq \frac{r^2 \epsilon}{\widehat{\lambda}_d^{(r)}- \widehat{\lambda}_{d+1}^{(r)} - 2r^2 \epsilon}.
    \end{align}
\end{thm}


The theorem establishes that for small values of $\epsilon$, the sine of the largest principal angle $\widehat{\Theta}_{\epsilon}^{(r)}$ can be bounded by either the linear bound \eqref{eq:upperbound1} or the rational bound \eqref{eq:upperbound2}. This implies that the winsorized PC subspace remains stable under minor contamination. However, it is important to note that the upper bound \eqref{eq:upperbound1} does not appear tight, as it grows linearly with the fraction of outliers. Similar statement can be made for SPCA. In particular, the bounds for SPCA are obtained by replacing 
${\widehat{\lambda}_j^{(r)}}/{r^2}$ with the $j$th largest eigenvalue of $\sum_{i=1}^n\xv_{i,0}\xv_{i,0}'/n\|\xv_{i,0}\|_2^2$.

To compare the perturbation bounds \eqref{eq:upperbound1} and \eqref{eq:upperbound2} in Theorem~\ref{thm:perturbation} and the lower bound of the breakdown point \eqref{eq:bp_win} in Theorem~\ref{thm:breakdown}, we use a data example. We fix the winsorization radius $r$ to be the median of the norms of the data points, i.e., $r = \mathrm{med}_i\{\|\xv_i\|\}$.

\begin{figure}[t]
\centering
    \includegraphics[width=\linewidth]{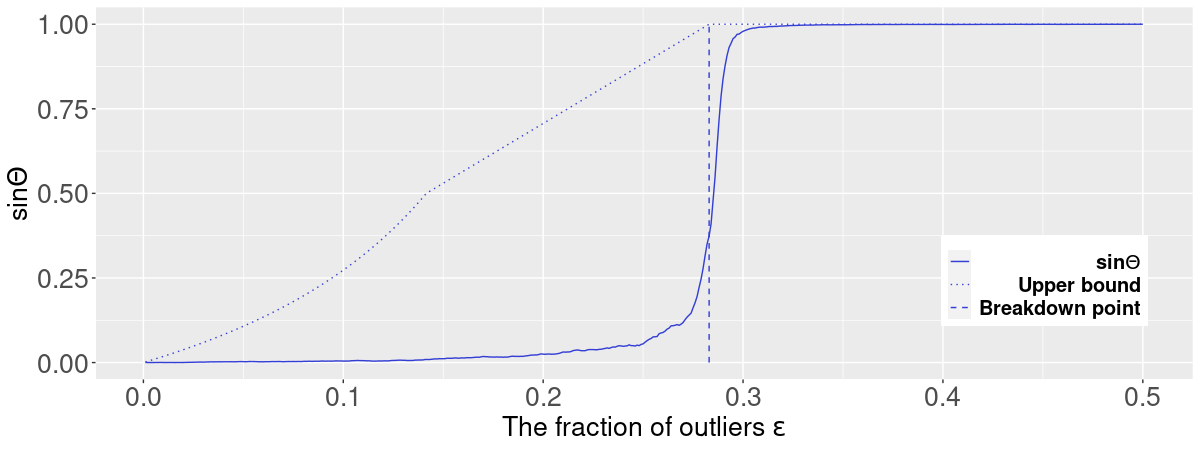}
    \caption{The largest principal angle $\Theta(\Vc_1^{(r)}(\Xv_\epsilon), \Vc_1^{(r)}(\Xv_0))$ and the perturbation bound versus contamination level $\epsilon$. 
    The solid line represents the observed largest principal angle, the dotted line represents the perturbation bound from Theorem~\ref{thm:perturbation}, and the vertical dashed line indicates the lower bound of the (weak) breakdown point from Theorem~\ref{thm:breakdown}.} \label{fig:rmed}
\end{figure}

For small contamination levels $\epsilon$, the perturbation bound closely follows the observed largest principal angle. This indicates that minor contamination leads to minor perturbation in the WPCA subspace, confirming the robustness of WPCA under small contamination. As $\epsilon$ increases, the perturbation bound becomes less sharp, overestimating the actual perturbation. The perturbation bound is conservative for larger contamination levels, suggesting that WPCA performs better in practice than the bound predicts.

In terms of the breakdown point, the principal angle remains relatively small until $\epsilon$ approaches the lower bound of the breakdown point in Theorem~\ref{thm:breakdown}. Once $\epsilon$ surpasses this breakdown point (indicated by the vertical dashed line), the principal angle rapidly increases towards $\pi/2$. The lower bound effectively predicts the actual breakdown point beyond which WPCA fails to recover the target subspace.

Consequently, WPCA demonstrates strong robustness to minor contamination, with the perturbation bound effectively predicting subspace deviation for small $\epsilon$. The breakdown point serves as a reliable threshold for subspace stability, effectively indicating when WPCA may fail to recover the target subspace. However, it is important to note that increasing the winsorization radius $r$ can reduce the robustness of WPCA, leading to higher perturbation bounds and lower breakdown points.


\section{CONCLUSION}
In terms of subspace recovery, our study demonstrates the accuracy of WPCA through concentration inequalities. We show that WPCA maintains consistency across a wide range of winsorization radii and performs well even in heavy-tailed distributions. 
Additionally, we demonstrate its consistency and scalability in high-dimensional settings with numerical examples.

Importantly, we introduce the concept of ``strong breakdown." Based on this concept, we reveal that WPCA exhibits higher resistance to contamination when compared to traditional PCA. However, we find that an excessively large winsorization radius negatively impacts subspace recovery, causing subspaces to diverge from the target, similar to the case of traditional PCA.

{
WPCA extends the applicability of traditional PCA by introducing robustness against anomalies. WPCA is most suited for the application areas that involve contaminated data. Examples include analyzing fMRI data for assessing brain connectivity \citep{lindquistStatisticalAnalysisFMRI2008a}, brain imaging visualization \citep{hanECAHighDimensionalElliptical2018b}, socio-economic studies for constructing indices like socio-economic status \citep{vyasConstructingSocioeconomicStatus2006a}. In system dynamics modeling, eigenvalue decomposition—closely related to PCA—serves as a key tool for extracting single-rate system dynamics from multirate sampled-data systems \cite{hanRobustSVDMade2024b}, in which robustness appears a valuable feature of the modeling.}

Future research could explore a spike model where population eigenvalues increase with the number of variables $p$. As observed in the numerical study in Section~\ref{sec:numeric}, we anticipate that the spiked eigenvalue model can be used to tighten the upper bounds in Theorem~\ref{thm:tail_ang} as $r$ and $p$ grow. Investigating the relationship between $r$, $p$, and eigenvalues in subspace recovery will provide valuable insights into the theoretical and practical applications of WPCA in high-dimensional settings.

\subsubsection*{Acknowledgements}
This research was supported by Basic Science Research Program through the National Research Foundation of Korea(NRF) funded by the Ministry of Education(RS-202400453397). This work was supported by Samsung Science and Technology Foundation under Project Number SSTF-BA2002-03.

\bibliographystyle{asa}
\bibliography{Winsor}

\clearpage
\appendix
\section{SUPPLEMENTARY MATERIAL}
\subsection{Technical details in Section~\ref{sec:stat_accuracy}}
\subsubsection{Covariance concentration and subguassian parametrization}
We provide concentration inequality for the sample covariance of a subgaussian random vector. We say that a random vector with zero mean $\xv \in \Rb^p$ is $\sigma$-subgaussian random vector if
\[
E[(\vv'\xv)^{2k}] \leq \frac{(2k)!}{2^k k!} \sigma^{2k}.
\]
for all $\vv \in S^{p-1}$ and $k=1,2,\dots$.

\begin{thm}[\cite{wainwrightHighDimensionalStatisticsNonAsymptotic2019b}]\label{thm:wainright}
    Let $\Xv = [\xv_1,\dots,\xv_n]' \in \Rb^{n\times p}$ be a data matrix whose rows are $i.i.d$ $\sigma$-subgaussian random vectors with zero mean and covariance matrix $\Sigmav$. Then the sample covariance $\widehat{\Sigmav} = \frac{1}{n}\Xv'\Xv$ satisfies the bound
    \[
        E[e^{\lambda\|\widehat{\Sigmav} - \Sigmav\|}] \leq e^{4p + 2^5 \frac{\lambda^2 \sigma^4}{n}}
    \]
    for all $|\lambda| < \frac{n}{2^3 \sigma^2}$.
\end{thm}
Here $\|\cdot\|$ implies the largest singular value of a matrix.
Using this theorem, we can have tail behavior and expectation bound of the sample covariance matrix as follows.
\begin{cor}\label{cor:concentration_cov}
    Let $\Xv = [\xv_1,\dots,\xv_n]' \in \Rb^{n\times p}$ be a data matrix whose rows are $i.i.d$ $\sigma$-subgaussian random vectors with zero mean and covariance matrix $\Sigmav$. Then,
    \begin{align*}
        P(\|\widehat{\Sigmav} - \Sigmav\| \geq t) 
        &\leq \exp\left(4p -\frac{n}{2} \left((\frac{t}{2^3\sigma^2})^2 \wedge \frac{t}{2^3\sigma^2}\right)\right).
    \end{align*}
    In other word, 
    \begin{align*}
    P(\|\widehat{\Sigmav} - \Sigmav\|\leq 2^3 \sigma^2 (\sqrt{\frac{2u + 8p}{n}} \vee \frac{2u + 8p}{n}))
    &\geq 1 - e^{-u}.
    \end{align*}
    The expectation is bounded as
    \begin{align*}
        E[\|\widehat{\Sigmav} - \Sigmav\|] 
        &= 2^3\sigma^2 E[\frac{\|\widehat{\Sigmav} - \Sigmav\|}{2^3\sigma^2}] \\
        &\leq 2^4 \sigma^2 (\frac{8p}{n} \vee \sqrt{\frac{8p}{n}}).
    \end{align*}
\end{cor}
\begin{proof}
    For any $t > 0$, for any $0 < \lambda < n$ by Chernoff bound, we have
    \begin{align*}
        P(\frac{\|\widehat{\Sigmav} - \Sigmav\|}{2^3 \sigma^2} \geq t) 
        &\leq E[e^{\frac{\lambda}{2^3 \sigma^2}\|\widehat{\Sigmav} - \Sigmav\|}] / e^{\lambda t}\\
        &\leq \exp(4p + \frac{1}{2n} \lambda^2 - \lambda t).
    \end{align*}
    By substitute the minimizer $\lambda = n(t \wedge 1)$, we have
    \begin{align*}
        P(\frac{\|\widehat{\Sigmav} - \Sigmav\|}{2^3 \sigma^2} \geq t) 
        &\leq \exp(4p + \frac{1}{2n} n^2(t\wedge 1)^2 - n (t\wedge 1) t)\\
        &\leq \exp(4p -\frac{n}{2} (t^2 \wedge t)).
    \end{align*}
    Additionally, if we set $t = \sqrt{\frac{2u + 8p}{n}} \vee \frac{2u + 8p}{n}$, we have $t^2 \wedge t = \frac{2u + 8p}{n}$ and
    \[
    P(\frac{\|\widehat{\Sigmav} - \Sigmav\|}{2^3 \sigma^2} \leq \sqrt{\frac{2u + 8p}{n}} \vee \frac{2u + 8p}{n}) \geq 1 - e^{-u} 
    \]
    For the expectation bound, 
    \begin{align*}
        E[\frac{\|\widehat{\Sigmav} - \Sigmav\|}{2^3 \sigma^2}]
        =&\int_0^\infty \exp\left(-\frac{n}{2}\left( t^2  \wedge t \right) + 4p
        \right) \wedge 1 dt \\
        =& \int_0^1 \exp\left(-\frac{n}{2} t^2 + 4p
        \right) \wedge 1 dt  + \int_1^\infty \exp\left(-\frac{n}{2}t + 4p
        \right) \wedge 1 dt \\
        \leq& \int_0^{\sqrt{\frac{8p}{n}}\wedge 1}1 dt + \int_{\sqrt{\frac{8p}{n}}\wedge 1}^1 \exp\left( -\frac{n}{2} t^2 + 4p \right) dt  \\
        &+ \int_1^{\frac{8p}{n}\vee 1} 1 dt + \int_{\frac{8p}{n}\vee 1}^\infty \exp\left( -\frac{n}{2} t + 4p \right) dt \\
        \leq &(\frac{8p}{n} \vee \sqrt{\frac{8p}{n}}) + \int_{\sqrt{\frac{8p}{n}}}^1 \exp\left( -\frac{n}{2} t^2 + 4p \right) dt \cdot I(\frac{8p}{n} < 1) \\
        &+ \frac{2}{n} \exp(-\frac{n}{2} (\frac{8p}{n} \vee 1) + 4p)\\
        \leq &(\frac{8p}{n} \vee \sqrt{\frac{8p}{n}}) 
        \\
        &\medskip + \int_{0}^{1 - \sqrt{\frac{8p}{n}}} \exp\left( -\frac{n}{2} (t + \sqrt{\frac{8p}{n}})^2 + 4p \right) dt \cdot I(\frac{8p}{n} < 1)
        + \frac{2}{n}\\
        \leq &(\frac{8p}{n} \vee \sqrt{\frac{8p}{n}}) 
        + \int_{0}^{1 - \sqrt{\frac{8p}{n}}} \exp\left(- t\sqrt{8np} \right) dt \cdot I(\frac{8p}{n} < 1)
        + \frac{2}{n}\\
        \leq &(\frac{8p}{n} \vee \sqrt{\frac{8p}{n}}) 
        + \frac{1}{\sqrt{8np}}\cdot I(\frac{8p}{n} < 1)
        + \frac{2}{n}\\
        \leq& 2 (\frac{8p}{n} \vee \sqrt{\frac{8p}{n}}).
    \end{align*}
    Consequently, We have
    \begin{align*}
        P(\|\widehat{\Sigmav} - \Sigmav\| \geq t) 
        &= P(\frac{\|\widehat{\Sigmav} - \Sigmav\|}{2^3\sigma^2} \geq \frac{t}{2^3\sigma^2})\\
        &\leq \exp\left(4p -\frac{n}{2} \left((\frac{t}{2^3\sigma^2})^2 \wedge \frac{t}{2^3\sigma^2}\right)\right),
    \end{align*}
    \begin{align*}
    P(\|\widehat{\Sigmav} - \Sigmav\|&\leq 2^3 \sigma^2 (\sqrt{\frac{2u + 8p}{n}} \vee \frac{2u + 8p}{n}))\\
    &= P(\frac{\|\widehat{\Sigmav} - \Sigmav\|}{2^3 \sigma^2} \\
    &\leq (\sqrt{\frac{2u + 8p}{n}} \vee \frac{2u + 8p}{n})) \leq 1 - e^{-u},
    \end{align*}
    and 
    \begin{align*}
        E[\|\widehat{\Sigmav} - \Sigmav\|] 
        &= 2^3\sigma^2 E[\frac{\|\widehat{\Sigmav} - \Sigmav\|}{2^3\sigma^2}] \\
        &\leq 2^4 \sigma^2 (\frac{8p}{n} \vee \sqrt{\frac{8p}{n}}).
    \end{align*}
\end{proof}

Here, we characterizes the subgaussian parameter $\sigma$ of $\xv^{(r)}$.
\begin{lem}\label{lem:subgaussian_xr}
    Assume that $\yv := \Sigmav^{-1/2}\xv$ be a $\sigma$-subgaussian. If $\yv = (y_1,\dots,y_p)'$ is not subgaussian, we denote $\sigma = \infty$. Then $\xv^{(r)}$ is $(\sqrt{\lambda_1}\sigma \wedge \sqrt{\frac{\lambda_1r^2}{\lambda_p p}})$-subgaussian. 
\end{lem}
\begin{proof}
    When $\Sigmav = \Vv \Lambdav \Vv'$ as described in Section~\ref{sec:concentration}, Without loss of generality, we assume that $\Vv = \Iv_p$. Let $s^2 = \|\xv\|_2^2 = \sum_{l=1}^p \lambda_l y_l^2$. For any $\wv \in S^{p-1}$, we have
    \begin{align*}
        E[(\wv' \xv^{(r)})^{2k}] 
        &= E\left[\left(\wv' \xv \left(1 \wedge \sqrt{\frac{r^2}{s^2}}\right)\right)^{2k}\right] \\
        &= E\left[\left((\wv' \sqrt{\Lambdav} \yv)^2 \left(1 \wedge \frac{r^2}{s^2}\right)\right)^k\right]\\
        &\leq E\left[\left(\left( \|\sqrt{\Lambdav}\wv\|_2 \frac{(\sqrt{\Lambdav}\wv)'}{\|\sqrt{\Lambdav}\wv\|_2} \yv\right)^2 \left(1 \wedge \frac{r^2}{\lambda_p\sum_{l=1}^p  y_l^2}\right)\right)^k\right]\\
        &= E\left[ \|\sqrt{\Lambdav}\wv\|_2^{2k}\left(\left(\frac{(\sqrt{\Lambdav}\wv)'}{\|\sqrt{\Lambdav}\wv\|_2} \yv\right)^2 \left(1 \wedge \frac{r^2}{\lambda_p\yv'\yv}\right)\right)^k\right]\\
        &\leq \lambda_1^k E\left[ \left(\left(\frac{(\sqrt{\Lambdav}\wv)'}{\|\sqrt{\Lambdav}\wv\|_2} \yv\right)^2 \left(1 \wedge \frac{r^2}{\lambda_p \yv'\yv}\right)\right)^k\right]\\
        &= \lambda_1^k E\left[ \left(y_1^2 \left(1 \wedge \frac{r^2}{\lambda_p \yv'\yv}\right)\right)^k\right],
        \end{align*}
        since for rotation $\Rv$ satisfying $\Rv \frac{\sqrt{\Lambdav}\wv}{\|\sqrt{\Lambdav}\wv\|_2} = (1,0,\dots,0)'$, $\Rv \yv \overset{\rm{d}}{=} \yv$ and $\yv'\yv \overset{\rm{d}}{=} (\Rv\yv)'(\Rv\yv)$. Note that $\frac{y_1^2}{\yv'\yv} \overset{\rm{d}}{=} \frac{z_1^2}{\zv'\zv}$ for standard gaussian random vector $\zv = (z_1,\dots,z_p)'$, thus the distribution of $\frac{y_1^2}{\yv'\yv}$ is the beta distribution with parameters $(\frac{1}{2},\frac{p}{2})$. Using this, we obtain
        \begin{align*}
        E[(\wv' \xv^{(r)})^{2k}] &\leq \lambda_1^k\left( E\left[ y_1^{2k}\right] \wedge E\left[\left(\frac{r^2y_1^2}{\lambda_p \yv'\yv}\right)^k\right] \right)\\
        &\leq \lambda_1^k\left( \frac{(2k)!}{2^k k!}\sigma^{2k} \wedge \frac{r^{2k}}{\lambda_p^k p^k} \frac{(2k)!}{2^k k!}  \right)\\
        &= \left( (\sqrt{\lambda_1}\sigma)^{2k} \wedge \left(\sqrt{\frac{\lambda_1r^2}{\lambda_p p}}\right)^{2k}   \right) \frac{(2k)!}{2^k k!}.
    \end{align*}
    Thus, $\xv^{(r)}$ is $(\sqrt{\lambda_1}\sigma \wedge \sqrt{\frac{\lambda_1r^2}{\lambda_p p}})$-subgaussian.   
\end{proof}

By applying Lemma~\ref{lem:subgaussian_xr} and Corollary~\ref{cor:concentration_cov} to $\xv^{(r)}$, we have the following theorem.
\begin{thm}[Covariance concentration]\label{thm:tail_mat}
 Let $\widehat{\Sigmav}^{(r)}_{\epsilon} = \frac{1}{n} \Xv_{\epsilon}'\Xv_{\epsilon}$ be the winsorized sample covariance matrix of contaminated data $\Xv_{\epsilon}$. 
 Then for any $r>0$, $p>1$, and $\epsilon \in [0,\frac{1}{2})$,
    \begin{align}\label{eq:tail_mat_expectation}
        E[\|\widehat{\Sigmav}^{(r)}_{\epsilon} - \Sigmav^{(r)}\|]\leq \epsilon r^2 + 2^4 \sigma_r^2 (\frac{8p}{n} \vee \sqrt{\frac{8p}{n}}).
    \end{align}
\end{thm}
\begin{proof}
    Note that $\Ic_\epsilon$ represents the indices of corrupted data. 
     Let $\widehat{\Sigmav}_{\rm{in}} = \frac{1}{(1-\epsilon)n} \sum_{i \not\in \Ic_{\epsilon}} \xv_{i,\epsilon}^{(r)}\xv_{i,\epsilon}^{(r)\prime}$ be the sample covariance of pure samples, and $\widehat{\Sigmav}_\out = \frac{1}{n\epsilon} \sum_{i \in\Ic_{\epsilon}} \xv_{i,\epsilon}^{(r)}\xv_{i,\epsilon}^{(r)\prime} $ be the sample covariance of outliers. Then, we have $\widehat{\Sigmav}^{(r)}_{\epsilon} = \epsilon\widehat{\Sigmav}_\out + (1-\epsilon)\widehat{\Sigmav}_{\rm{in}}$. Note that $\|(\widehat{\Sigmav}_\out-\Sigmav^{(r)})\| = \eta_1(\widehat{\Sigmav}_\out-\Sigmav^{(r)}) \vee \eta_1(-\widehat{\Sigmav}_\out+\Sigmav^{(r)})$ where $\eta_1(\cdot)$ is the largest eigenvalue. Thus we have $\|(\widehat{\Sigmav}_\out-\Sigmav^{(r)})\| \leq r^2$
    \begin{align*}
        E[\|\widehat{\Sigmav}^{(r)}_{\epsilon} - \Sigmav^{(r)}\|] &= E[\|\epsilon\widehat{\Sigmav}_\out + (1-\epsilon)\widehat{\Sigmav}_{\rm{in}} - \Sigmav^{(r)}\|] \\
        &\leq E[\|\epsilon(\widehat{\Sigmav}_\out-\Sigmav^{(r)})\|] + E[\|(1-\epsilon)(\widehat{\Sigmav}_{\rm{in}} - \Sigmav^{(r)})\|]\\
        &\leq \epsilon r^2 + (1-\epsilon)E[\|\widehat{\Sigmav}_{\rm{in}} - \Sigmav^{(r)}\|]\\
        &\leq \epsilon r^2 + 2^4 (1-\epsilon) \sigma_r^2 (\frac{8p}{(1-\epsilon)n} \vee \sqrt{\frac{8p}{(1-\epsilon)n}})\\
        &\leq \epsilon r^2 + 2^4 \sigma_r^2 (\frac{8p}{n} \vee \sqrt{\frac{8p}{n}}).
    \end{align*}
\end{proof}
The concentration of a sample covariance is highly related to concentration of subspaces spanned by eigenvectors of the sample covariance matrix. We provide a lemma based on the variant of Davis-Kahan theorem by \cite{yuUsefulVariantDavis2015a}.
    \begin{lem}\label{lem:kahan}
        Let $\Sigmav, \widehat{\Sigmav} \in \Rb^{p\times p}$ be symmetric, with eigenvalues $\lambda_1 \geq \dots \geq \lambda_p$ and $\widehat{\lambda}_1 \geq \dots \geq \widehat{\lambda}_p$, respectively. Assume that $\lambda_d > \lambda_{d+1}$. Let $\Vv = (\vv_1,\dots,\vv_d) \in \Rb^{p \times d}$ and $\widehat{\Vv} = (\widehat{\vv}_1,\dots,\widehat{\vv}_d) \in \Rb^{p\times d}$ have orthonormal eigenvector columns satisfying $\Sigmav \vv_j =\lambda_j \vv_j$ and $\widehat{\Sigmav} \widehat{\vv}_j = \widehat{\lambda}_j \widehat{\vv}_j$ for $j=1,\dots,d$. Let $\Vc$ and $\widehat{\Vc}$ be the subspaces in $\Rb^p$ spanned by the columns of $\Vv$ and $\widehat{\Vv}$, respectively. Let $\Theta \in [0,\frac{\pi}{2}]$ be the largest principal angle between $\Vc$ and $\widehat{\Vc}$. Then,
        \begin{align*}
            \sin \Theta \leq \frac{2 \|\widehat{\Sigmav} - \Sigmav\|}{\lambda_d -\lambda_{d+1}}
        \end{align*}
    \end{lem}
    \begin{proof}
        The details of this proof is given by replacing Frobenius norm with operator 2-norm in \cite{yuUsefulVariantDavis2015a}. Let $\Lambdav = \diag(\lambda_1,\dots,\lambda_d)$, $\Lambdav_\perp = \diag(\lambda_{d+1},\dots,\lambda_{p})$ and $\widehat{\Lambdav} = \diag(\widehat{\lambda}_1,\dots,\widehat{\lambda}_d)$. Then,
        \begin{align*}
            \Ov &= \widehat{\Sigmav}\widehat{\Vv} - \widehat{\Vv}\widehat{\Lambdav}\\
            &= (\widehat{\Sigmav} - \Sigmav + \Sigmav)\widehat{\Vv} - \widehat{\Vv}(\widehat{\Lambdav} - \Lambdav + \Lambdav)\\
            &= \Sigmav\widehat{\Vv} - \widehat{\Vv}\Lambdav + (\widehat{\Sigmav} - \Sigmav)\widehat{\Vv} - \widehat{\Vv}(\widehat{\Lambdav} - \Lambdav)
        \end{align*}
        where $\Ov$ is a matrix with a proper size whose elements are zero. Thus, we have
        \begin{align*}
            \|\Sigmav\widehat{\Vv} - \widehat{\Vv}\Lambdav\| &= \|(\widehat{\Sigmav} - \Sigmav)\widehat{\Vv} - \widehat{\Vv}(\widehat{\Lambdav} - \Lambdav)\|\\
            &\leq\|(\widehat{\Sigmav} - \Sigmav)\widehat{\Vv}\| + \|\widehat{\Vv}(\widehat{\Lambdav} - \Lambdav)\|\\
            &\leq\|(\widehat{\Sigmav} - \Sigmav)\|\|\widehat{\Vv}\| + \|\widehat{\Vv}\|\|(\widehat{\Lambdav} - \Lambdav)\|\\
            &\leq\|(\widehat{\Sigmav} - \Sigmav)\| + \|(\widehat{\Lambdav} - \Lambdav)\|\\
            &\leq 2\|(\widehat{\Sigmav} - \Sigmav)\|,
        \end{align*}
        since $\|\widehat{\Vv}\| \leq 1$, and $ \|(\widehat{\Lambdav} - \Lambdav)\| \leq \|(\widehat{\Sigmav} - \Sigmav)\|$ by Weyl's inequality. Meanwhile, since $\|\Vv_\perp\|=1$, we have
        \begin{align*}
            \|\Sigmav\widehat{\Vv} - \widehat{\Vv}\Lambdav\| &= \|\Vv_\perp'\| \|\Sigmav\widehat{\Vv} - \widehat{\Vv}\Lambdav\|\\
            &\geq \|\Vv_\perp' (\Sigmav\widehat{\Vv} - \widehat{\Vv}\Lambdav)\|\\
            &= \|\Vv_\perp'\Sigmav\widehat{\Vv} - \Vv_\perp'\widehat{\Vv}\Lambdav\|\\
            &= \|\Lambdav_\perp\Vv_\perp'\widehat{\Vv} - \Vv_\perp'\widehat{\Vv}\Lambdav\|\\
            &\geq \|\Vv_\perp'\widehat{\Vv}\Lambdav\| - \|\Lambdav_\perp\Vv_\perp'\widehat{\Vv}\|\\
            &\geq \lambda_d\|\Vv_\perp'\widehat{\Vv}\|- \lambda_{d+1}\|\Vv_\perp'\widehat{\Vv}\|\\
            &= (\lambda_d - \lambda_{d+1})\|\Vv_\perp'\widehat{\Vv}\|.
        \end{align*}
        Here, the last inequality holds since 
        \begin{align*}
            \|\Vv_\perp'\widehat{\Vv}\Lambdav\| &= \sup_{\|\uv\|_2 =1} \|\Lambdav\widehat{\Vv}'\Vv_\perp \uv\|_2\\
            &= \sup_{\|\uv\|_2 =1} \|\Lambdav\frac{\widehat{\Vv}'\Vv_\perp \uv}{\|\widehat{\Vv}'\Vv_\perp \uv\|_2}\|\widehat{\Vv}'\Vv_\perp \uv\|_2\|_2\\
            &= \sup_{\|\uv\|_2 =1} \|\Lambdav\frac{\widehat{\Vv}'\Vv_\perp \uv}{\|\widehat{\Vv}'\Vv_\perp \uv\|_2}\|_2 \|\widehat{\Vv}'\Vv_\perp \uv\|_2\\
            &\geq \sup_{\|\uv\|_2 =1} \lambda_d \|\widehat{\Vv}'\Vv_\perp \uv\|_2\\
            &= \lambda_d \|\Vv_\perp'\widehat{\Vv}\|.
        \end{align*}
        Note that 
        \begin{align*}
            \|\Vv_\perp'\widehat{\Vv}\|^2 
            &= \|\widehat{\Vv}'\Vv_\perp\Vv_\perp'\widehat{\Vv}\|\\
            &= \|\widehat{\Vv}'(\Iv_p - \Vv\Vv')\widehat{\Vv}\|\\
            &= \|\widehat{\Vv}'\widehat{\Vv} - \widehat{\Vv}'\Vv\Vv'\widehat{\Vv}\|\\
            &= \|\Iv_d - \widehat{\Vv}'\Vv\Vv'\widehat{\Vv}\|\\
            &= \|\Pv\Pv' - \Pv \cos^2 \Pv'\|\\
            &= \sin^2 \Theta.
        \end{align*}
        where $\widehat{\Vv}'\Vv\Vv'\widehat{\Vv} = \Pv \diag(\cos^2(\theta_1),\dots,\cos^2(\theta_d)) \Pv'$ is the eigen decomposition and the entries of the diagonal matrix $\diag(\cos^2(\theta_1),\dots,\cos^2(\theta_d))$ represent the squares of the cosines of the principal angles between $\Vc$ and $\widehat{\Vc}$.
    \end{proof}

\subsubsection{Proof of Theorem~\ref{thm:tail_ang}}
\begin{proof}
By combining Theorem~\ref{thm:tail_mat} and Lemma~\ref{lem:kahan}, we have
\begin{align*}
        E\left[\sin \Theta_{\epsilon}^{(r)}\right] 
        &\leq E\left[\frac{2\|\widehat{\Sigmav}^{(r)}_{\epsilon} - \Sigmav^{(r)}\|}{\lambda_d^{(r)} -\lambda_{d+1}^{(r)}}\right]\\
        &\leq \frac{2 r^2\epsilon }{\lambda_d^{(r)} -\lambda_{d+1}^{(r)}} + \frac{2^5 \sigma_r^2 (\frac{8p}{n} \vee \sqrt{\frac{8p}{n}})}{\lambda_d^{(r)} -\lambda_{d+1}^{(r)}}.
\end{align*}
Using Lemma~\ref{lem:subgaussian_xr}, we conclude the results of Theorem~\ref{thm:tail_ang}.
\end{proof}



\subsubsection{Proof of Corollary~\ref{thm:tail_ang_cor}}
\begin{proof}
We first provide asymptotic properties about winsorized eigenvalue $\lambda_j{(r)}$ and the radius $r$.
\begin{lem}\label{lem:asymp_eigen}
    For increasing $p$ and $r = p^{\frac{1}{2}+ \beta}$, Let $\lambda_j^{(r)}$ denote the $j$th largest eigenvalue of $\Cov(\xv^{(r)})$. Then,
    \begin{align*}
        \lambda_j^{(r)} = \Omega(p^{2 (\beta \wedge 0)}), \text{ and }
        \frac{r^2}{\lambda_d^{(r)} - \lambda_{d+1}^{(r)}} = \Omega(p^{1 + 2(\beta \vee 0)}).
    \end{align*}
\end{lem}
\begin{proof}[proof of lemma]
    Let $\Sigmav^{-1/2}\xv = \yv = [y_1,\dots,y_p]'$. By \cite{kingmanRandomSequencesSpherical1972a}, there exists a non-negative random variable $w$ such that $y_j = \sqrt{w}z_j$ where $z_1,z_2,\dots$ are i.i.d. standard normal random variables. Without loss of generality, we assume that 
    $\lambda_j^{(r)} = E[\lambda_j y_j^2 (1 \wedge \frac{r^2}{s^2})]$ with $s^2 = \sum_{l=1}^p \lambda_l y_l^2$.
    Note that $\lambda_j y_j^2 (1 \wedge \frac{r^2}{s^2}) \leq \lambda_j y_j^2$ with $E[\lambda_j y_j^2] = \lambda_j$. If $\beta > 0$, by Dominant Convergence Theorem, we have
    \begin{align*}
        \lim_{p}\lambda_j^{(r)} &= \lambda_j E[ y_j^2 \wedge \lim_{p}\frac{p^{1+2\beta}y_j^2}{\sum_{l=1}^p \lambda_l y_l^2})]\\
        &= \lambda_j E[ y_j^2] = \lambda_j.
    \end{align*}
    Similarly, if $\beta = 0$, 
    \begin{align*}
        \lim_{p}\lambda_j^{(r)} &= \lambda_j E[ y_j^2 \wedge \lim_{p}\frac{py_j^2}{\sum_{l=1}^p \lambda_l y_l^2})]\\
        &= \lambda_j E[ w \wedge \frac{1}{\lambda}) ].
    \end{align*}

    In case of $\beta < 0$, 
    \begin{align*}
        \frac{p}{r^2}\lambda_j^{(r)} &= \frac{p}{r^2} E[\lambda_j y_j^2 (1 \wedge \frac{r^2}{s^2})]\\
        &\leq E[\frac{p \lambda_j z_j^2}{\lambda \sum_{l=1}^p z_l^2})]\\
        &= \lambda_j/\lambda.
    \end{align*}
    Conversely, by Fatou's lemma, we have
    \begin{align*}
        \liminf_{p}\frac{p}{r^2}\lambda_j^{(r)} &\geq  E[\liminf_{p}\frac{p}{r^2}\lambda_j y_j^2 (1 \wedge \frac{r^2}{s^2})]\\
        &= E[\frac{\lambda_jz_j^2}{\lambda})] = \frac{\lambda_j}{\lambda}.
    \end{align*}
    Thus we have $\lim_p \frac{p}{r^2}\lambda_j^{(r)} = \lambda_j/\lambda$. Consequently,
    \begin{align*}
        \lambda_j^{(r)} &= \Omega(p^{2 (\beta \wedge 0)}),\\
        \frac{r^2}{\lambda_d^{(r)} - \lambda_{d+1}^{(r)}} &= \Omega(p^{1 + 2(\beta \vee 0)}).
    \end{align*}
\end{proof}
By applying this lemma to Theorem~\ref{thm:tail_ang}, we have the conclusion.
\end{proof}

\subsection{Technical details in Section~\ref{sec:robustness}}
\subsubsection{Proof of Theorem~\ref{thm:breakdown_PCA}}
\begin{proof}
    By Propositioin 3 in \cite{hanRobustSVDMade2024b}, $\bp(\Vc_d;\Xv) = \frac{1}{n}$ holds. We show $\overline{\bp}(\Vc_d ;\Xv)  \leq \frac{d}{n}$. Since $2d \leq p$, we can find $d$ orthonormal vectors $\widehat{\wv}_1,\dots,\widehat{\wv}_d$ belonging to $\widehat{\Wc}_d := (\Vc_d(\Xv))^{\perp}$, where $(\Vc_d(\Xv))^{\perp}$ is the complemented subspace of $\Vc_d(\Xv)$. For $c>0$, let $\Zv_c = (c\widehat{\wv}_1,\dots,c\widehat{\wv}_d,\xv_{d+1},\dots,\xv_n)'$ be the contaminated data, and $\widehat{\Sigmav}(c) = \frac{1}{n}\Zv_c'\Zv_c = \frac{1}{n}\sum_{j=1}^p \widehat{\lambda}_{j,c} \widehat{\vv}_{j,c}\widehat{\vv}_{j,c}'$ be the contaminated sample covariance matrix satisfying $\widehat{\lambda}_{1,c} \geq \dots, \geq \widehat{\lambda}_{p,c}$. 
    Let $\widehat{\Vc}_d(c) = \Vc_d(\Zv_c)$ be the PC subspace with $\Zv_c$, $\widetilde{\Vc}_d(c)$ be the $d$-dimensional subspace spanned by $\widehat{\wv}_1,\dots,\widehat{\wv}_d$, and $\widetilde{\Sigmav}_d(c) = \frac{1}{n}\sum_{j=1}^d c^2 \widehat{\wv}_j\widehat{\wv}_j'$. By the variant of Davis-Kahan theorem by \cite{yuUsefulVariantDavis2015a}, we have
    \begin{align*}
    \|\sin \Theta(\widehat{\Vc}_d(c),\widetilde{\Vc}_d(c))\|_{\rm{F}}
    &\leq \frac{2 \|\widehat{\Sigmav}(c) - \widetilde{\Sigmav}_d(c)\|_{\rm{F}}}{c^2}\\
    &\leq \frac{2 \|\frac{1}{n}\sum_{i=d+1}^n \xv_i \xv_i'\|_{\rm{F}}}{c^2}\\
    &\leq \frac{2M}{c^2},
\end{align*}
for some $M$ which does not depend on $c$. For two $d$-dimensional subspace $\Vc$ and $\Uc$, let $\sin \overrightarrow{\Theta}(\Vc,\Uc) := \diag(\sin \theta_1(\Vc,\Uc),\dots,\sin \theta_d(\Vc,\Uc))$ be the diagonal matrix whose entries are sine of $d$ principal angles. Note that $\widetilde{\Vc}_d(c)$ is orthogonal to $\widehat{\Vc}_d:= \Vc_d(\Xv)$, indicating that $\|\sin \overrightarrow{\Theta}(\widehat{\Vc}_d,\widetilde{\Vc}_d(c))\|_{\rm{F}} = d$. Thus we obtain
    \begin{align*}
    \|\sin \overrightarrow{\Theta}(\widehat{\Vc}_d,\widehat{\Vc}_d(c))\|_{\rm{F}} 
    &\geq \|\sin \overrightarrow{\Theta}(\widehat{\Vc}_d,\widetilde{\Vc}_d(c))\|_{\rm{F}} - \|\sin \overrightarrow{\Theta}(\widehat{\Vc}_d(c),\widetilde{\Vc}_d(c))\|_{\rm{F}}\\
    &\geq d - \frac{2M}{c^2}.
    \end{align*}
    The triangle inequality holds since $\|\sin \overrightarrow{\Theta}(\Vc,\Uc)\|_{\rm{F}} = \|\Pi_{\Vc} - \Pi_{\Uc}\|_{\rm{F}}/\sqrt{2}$ where $\Pi_{\Vc},$ and $\Pi_{\Uc}$ are the projection matrices of $\Vc$ and $\Uc$, respectively. Thus, for any $\epsilon >0$, we can find a contaminated data $\Zv_c$ such that
    $\|\sin \overrightarrow{\Theta}(\widehat{\Vc}_d,\widehat{\Vc}_d(c))\|_{\rm{F}} \geq d - \epsilon$ by taking a sufficiently large $c$. It means $\overline{\bp}(\Vc_d ;\Xv)  \leq \frac{d}{n}$.

    $\|\Pi\uv_0\| \leq \epsilon$, and $\|\vv_0 - \wv_0 \| \leq \epsilon$ for some unit vector $\wv_0 \in \widetilde{\Vc}_d^{\perp}$.

    On the other hand, we show $\overline{\bp}(\Vc_d ;\Xv)  > \frac{d-1}{n}$. For a $d$-dimensional subspace $\Vc_d$, let $\Pi_{\Vc_d}$ be the projection matrix onto $\Vc_d$. Assume that, for any $\epsilon \in [0,1)$, there exists the contaminated data $\Zv$ satisfying 
    $\|\Pi_{\Vc_d(\Zv)}\Pi_{\Vc_d(\Xv)}\| \leq \|\Pi_{\Vc_d(\Zv)}\Pi_{\Vc_d(\Xv)}\|_{\rm{F}} < \epsilon$ with only $d-1$ contaminated data points. Note that this assumption is equivalent to $\overline{\bp}(\Vc_d;\Xv) \leq \frac{d-1}{n}$. Without loss of generality, let $\Zv = (\zv_1,\dots,\zv_{d-1},\xv_d,\dots,\xv_n)'$. Let $\widetilde{\Vc}_d := \Vc_d(\Zv)$, $\widehat{\Vc}_d := \Vc_d(\Xv)$, $\widetilde{\Pi} = \Pi_{\widetilde{\Vc}_d}$, and $\widehat{\Pi} = \Pi_{\widehat{\Vc}_d}$. We can find unit vectors $\vv_0 \in \widehat{\Vc}_d$ and $\uv_0 \in \widetilde{\Vc}_d$ such that $\xv_i'\vv_0 =0$ and $\zv_i'\uv_0=0$ for all $i=1,\dots,d-1$ since $\widehat{\Vc}_d$ and $\widetilde{\Vc}_d$ are $d$-dimension and spanning subspace of $(d-1)$ data points are at most $(d-1)$-dimension. Then, 
    \begin{align*}
        \widehat{\lambda}_d 
        = \min_{\vv \in \widehat{\Vc}_d, \|\vv\|=1} \vv' \Xv'\Xv \vv /n 
        \leq \vv_0'\Xv'\Xv\vv_0 = \vv_0'\Xv_{-}'\Xv_{-}\vv_0,
    \end{align*}
    where $\Xv_{-} = [\xv_{d+1},\dots, \xv_n]' \in \Rb^{(n-d) \times p}$. Since $\|\widehat{\Pi}\widetilde{\Pi}\|_{\rm{F}} < \epsilon$, we know that
    \begin{align*}
        \|\widehat{\Pi} \uv_0\| &=  \|\widehat{\Pi} \widetilde{\Pi} \uv_0\| \leq \epsilon,\\
        \|\widetilde{\Pi} \vv_0\| &=  \|\widetilde{\Pi} \widehat{\Pi} \vv_0\| \leq \epsilon,
    \end{align*}
    and for $\wv_0 := \frac{(\Iv_p - \widetilde{\Pi})\vv_0}{\|(\Iv_p - \widetilde{\Pi})\vv_0\|} \in \widetilde{\Vc}_d^{\perp}$, 
    \begin{align*}
        \|\vv_0 - \wv_0\|^2 &=  2 - 2\|(\Iv_p - \widetilde{\Pi})\vv_0\|\\
        &\leq 2 - 2(\| \vv_0 \| - \|\widetilde{\Pi}\vv_0\|)\\
        &\leq 2\epsilon.
    \end{align*}
    Let $\widetilde{\lambda}_j$ be the $j$th largest eigenvalue of $\frac{1}{n}\Zv'\Zv$. Then, we have
    \begin{align*}
        \widetilde{\lambda}_d 
        &=  \min_{\vv \in \widetilde{\Vc}_d, \|\vv\|=1} \vv'\Zv'\Zv \vv /n\\
        &\leq \uv_0'\Zv'\Zv \uv_0 /n\\
        &= \uv_0'\Xv_{-}'\Xv_{-} \uv_0 /n\\
        &\leq \uv_0'\Xv'\Xv \uv_0 /n\\
        &= \uv_0'(\widehat{\Pi} + \widehat{\Pi}^{\perp})\Xv'\Xv(\widehat{\Pi} + \widehat{\Pi}^{\perp}) \uv_0 /n\\
        &= \uv_0'\widehat{\Pi}^{\perp}\Xv'\Xv\widehat{\Pi}^{\perp}\uv_0 /n + 2\uv_0\widehat{\Pi}\Xv'\Xv\widehat{\Pi}^{\perp}\uv_0 /n + \uv_0\widehat{\Pi}\Xv'\Xv\widehat{\Pi}\uv_0 /n\\
        &\leq \|\widehat{\Pi}^{\perp}\Xv'\Xv\widehat{\Pi}^{\perp}/n\| + 2\|\widehat{\Pi} \uv_0\|\cdot \|\Xv'\Xv /n\|\cdot  \|\widehat{\Pi}^{\perp}\uv_0\| + \|\widehat{\Pi} \uv_0'\|^2 \cdot \|\Xv'\Xv /n\|\\
        &\leq \widehat{\lambda}_{d+1} + 3\widehat{\lambda}_1\epsilon,
        \end{align*}
        where $\widehat{\Pi}^{\perp} = \Iv_p - \widehat{\Pi}$. Meanwhile,
        \begin{align*}
        \widetilde{\lambda}_{d+1} 
        &=  \max_{\vv \in \widetilde{\Vc}_d^{\perp}, \|\vv\|=1} \vv'\Zv'\Zv \vv /n\\
        &\geq  \max_{\vv \in \widetilde{\Vc}_d^{\perp}, \|\vv\|=1} \vv'\Xv'_{-}\Xv_{-} \vv /n\\
        &\geq  \wv_0'\Xv'_{-}\Xv_{-} \wv_0 /n\\
        &=  (\wv_0-\vv_0 + \vv_0)'\Xv'_{-}\Xv_{-} (\wv_0-\vv_0 + \vv_0) /n\\
        &=  \vv_0'\Xv'_{-}\Xv_{-}\vv_0 /n + 2(\wv_0-\vv_0)'\Xv'_{-}\Xv_{-}\vv_0 /n\\
        & \hspace{0.5cm} + (\wv_0-\vv_0)'\Xv'_{-}\Xv_{-} (\wv_0-\vv_0) /n\\
        &\geq  \widehat{\lambda}_d - 4\widehat{\lambda}_1 \epsilon - 4 \widehat{\lambda}_1 \epsilon^2\\
        &\geq  \widehat{\lambda}_d - 8\widehat{\lambda}_1 \epsilon.
    \end{align*}
    It means, $\widetilde{\lambda}_d - \widetilde{\lambda}_{d+1} \leq  \widehat{\lambda}_{d+1} - \widehat{\lambda}_d + 11\widehat{\lambda}_1 \epsilon$. Since $\epsilon$ is arbitrary, we can find a $\Zv$ such that $\widetilde{\lambda}_d - \widetilde{\lambda}_{d+1} \leq  \widehat{\lambda}_{d+1} - \widehat{\lambda}_d + 11\widehat{\lambda}_1 \epsilon < 0$ by taking sufficiently small $\epsilon$. It is a contradiction. Thus $\overline{\bp}(\Vc_d ;\Xv) = \frac{d}{n}$.
\end{proof}

\subsubsection{Proof of Theorem~\ref{thm:breakdown}}
\begin{proof}
Denote $\bp(\Vc_d^{(r)};\Xv) = \epsilon$. It implies that for any small $\epsilon_0 \in (0,1)$, there exists $\Xv_{\epsilon}$ such that $|I_{\epsilon}| := |\{1\leq i \leq n : \xv_i = \xv_{i,\epsilon}\}| = (1-\epsilon)n$, and $\sin \widehat{\Theta} := \sin \Theta\left(\Vc_d^{(r)}(\Xv_{\epsilon}),\Vc_d^{(r)}(\Xv)\right) \geq 1- \epsilon_0$. By Theorem~\ref{thm:perturbation},
\begin{align*}
    1 - \epsilon_0 &\leq \sin \widehat{\Theta} \leq \frac{2r^2 \epsilon}{\widehat{\lambda}_d^{(r)} - \widehat{\lambda}_{d+1}^{(r)}}.
\end{align*}
Since $\epsilon_0$ is arbitrarily small, we have
\begin{align*}
    \bp(\Vc_d^{(r)};\Xv) = \epsilon \geq \frac{\widehat{\lambda}_d^{(r)} - \widehat{\lambda}_{d+1}^{(r)}}{2r^2}.
\end{align*}

For strong breakdown point, let $\Pi \in \Rb^{p \times p}$ be the projection matrix of $\Vc_d(\Xv^{(r)})$ and $\Pi^\perp = (\Iv_p - \Pi)$. Denote $\overline{\bp}(\Vc_d^{(r)} ;\Xv) = \epsilon$. By definition of $\overline{\bp}(\Vc_d^{(r)} ;\Xv)$, for any $\epsilon_0 \in (0,1)$, there exists $\Xv_{\epsilon} = [\xv_{1,\epsilon},\dots,\xv_{n,\epsilon}]' \in \Rb^{n\times p}$ such that $|I_{\epsilon}| \coloneqq |\{1\leq i \leq n : \xv_{i,\epsilon} = \xv_i\}| = (1-\epsilon)n$, and eigen bases $\vv_{j,\epsilon}$ corresponding to the $j$th largest eigenvalue of $\frac{1}{n}\left(\Xv_{\epsilon}^{(r)\prime}\Xv_{\epsilon}^{(r)}\right)$ and $\|\Pi \vv_{j,\epsilon}\|_2 \leq \epsilon_0$ for $j=1,\dots,d$. Then, for each $j=1,\dots,d$ we have
\begin{align*}
    \vv_{j,\epsilon}^{\perp \prime} \left(\frac{1}{n} \Xv_{\epsilon}^{(r)\prime}\Xv_{\epsilon}^{(r)}\right) \vv_{j,\epsilon}
    &= \vv_{j,\epsilon}'\left(\frac{1}{n} \sum_{i \not\in I_{\epsilon}} \xv_{i,\epsilon}^{(r)} \xv_{i,\epsilon}^{(r)\prime}\right)\vv_{j,\epsilon} \\
    & \hspace{0.5cm}+ \vv_{j,\epsilon}' \left(\frac{1}{n} \sum_{i \in I_{\epsilon}} \xv_i^{(r)} \xv_i^{(r)\prime}\right)\vv_{j,\epsilon}\\
   &\leq  r^2\epsilon \\
   & \hspace{0.5cm}+ \vv_{j,\epsilon}' (\Pi^\perp + \Pi) \left(\frac{1}{n} \sum_{i \in I_{\epsilon}} \xv_i^{(r)} \xv_i^{(r)\prime}\right)(\Pi^\perp + \Pi) \vv_{j,\epsilon}\\
   &\leq  r^2\epsilon + 3 r^2(1-\epsilon)\epsilon_0 \\
   & \hspace{0.5cm} + \vv_{j,\epsilon}'\left( \Pi^\perp \left(\frac{1}{n} \sum_{i \in I_{\epsilon}}\xv_i^{(r)} \xv_i^{(r)\prime}\right) \Pi^\perp \right)\vv_{j,\epsilon}\\
   &\leq  r^2\epsilon + 3 r^2(1-\epsilon)\epsilon_0 \\
   & \hspace{0.5cm}+ \vv_{j,\epsilon}'\left( \Pi^\perp \left(\frac{1}{n} \Xv^{(r)\prime}\Xv^{(r)}\right) \Pi^\perp \right)\vv_{j,\epsilon}.
\end{align*}
Meanwhile,
\begin{align*}
    \vv_{j,\epsilon}' &\left(\frac{1}{n} \Xv_{\epsilon}^{(r)\prime}\Xv_{\epsilon}^{(r)}\right) \vv_{j,\epsilon}\\
    &= \eta_{j}\left(\frac{1}{n} \Xv_{\epsilon}^{(r)\prime}\Xv_{\epsilon}^{(r)}\right)\\
    &= \eta_j\left( \frac{1}{n}\left( \sum_{i =1}^n\xv_i^{(r)}\xv_i^{(r)\prime} \right) + \frac{1}{n} \left(\sum_{i \not\in I_{\epsilon}} \xv_{i,\epsilon}^{(r)}\xv_{i,\epsilon}^{(r)\prime} - \xv_{i}^{(r)}\xv_{i}^{(r)\prime}\right)\right)\\
    &\geq\eta_j\left(\frac{1}{n} \Xv^{(r)\prime}\Xv^{(r)}\right) - r^2\epsilon,
\end{align*}
where $\eta_j(\Av)$ is the $j$th largest eigenvalue of $\Av \in \Rb^{p\times p}$. Thus we have 
\begin{align*}
    2r^2\epsilon + 3r^2 (1-\epsilon)\epsilon_0 
    &\geq \eta_j\left(\frac{1}{n} \Xv^{(r)\prime}\Xv^{(r)}\right) - \vv_{j,\epsilon}'\left( \Pi^\perp \left(\frac{1}{n} \Xv^{(r)\prime}\Xv^{(r)}\right) \Pi^\perp \right)\vv_{j,\epsilon}\\
    &=\widehat{\lambda}_j^{(r)} - \vv_{j,\epsilon}'\left( \Pi^\perp \left(\frac{1}{n} \Xv^{(r)\prime}\Xv^{(r)}\right) \Pi^\perp \right)\vv_{j,\epsilon}\\
    &=\widehat{\lambda}_j^{(r)} - \tr\left( \Pi^\perp \left(\frac{1}{n} \Xv^{(r)\prime}\Xv^{(r)} \right) \Pi^\perp \vv_{j,\epsilon}\vv_{j,\epsilon}' \right),
\end{align*}
for $j=1,\dots,d$. Since $(\vv_{1,\epsilon},\dots,\vv_{d,\epsilon})'(\vv_{1,\epsilon},\dots,\vv_{d,\epsilon}) = \Iv_d$, for every $d_0 = 1,\dots,d$,
\begin{align*}
    &d_0(2r^2\epsilon + 3r^2 (1-\epsilon)\epsilon_0 )\\
    &\geq \sum_{j=1}^{d_0}\left(\widehat{\lambda}_j^{(r)} - \tr\left( \Pi^\perp \left(\frac{1}{n} \Xv^{(r)\prime}\Xv^{(r)} \right) \Pi^\perp \vv_{j,\epsilon}\vv_{j,\epsilon}' \right)\right)\\
    &\geq \sum_{j=1}^{d_0}\widehat{\lambda}_j^{(r)} - \sup_{u_1,\dots,u_{d_0}}\sum_{j=1}^{d_0}\tr\left( \Pi^\perp \left(\frac{1}{n} \Xv^{(r)\prime}\Xv^{(r)} \right) \Pi^\perp \uv_j\uv_j' \right)\\
    &= \sum_{j=1}^{d_0}\widehat{\lambda}_j^{(r)} - \sup_{u_1,\dots,u_{d_0}}\tr\left( \Pi^\perp \left(\frac{1}{n} \Xv^{(r)\prime}\Xv^{(r)} \right) \Pi^\perp \sum_{j=1}^{d_0}\uv_j\uv_j' \right)\\
    &\geq \sum_{j=1}^{d_0}\widehat{\lambda}_j^{(r)} - \sum_{j=1}^{d_0}\widehat{\lambda}_{d+j}^{(r)}.
\end{align*}
where the supremum is taken over all possible $\uv_1,\dots,\uv_{d_0}$ satisfying orthonormality, $(\uv_1,\dots,\uv_{d_0})'(\uv_1,\dots,\uv_{d_0}) = \Iv_{d_0}$. Here, the last inequality is obtained from Von Neumann's trace inequality. Since $\epsilon_0$ is arbitrary, we have
\begin{align*}
    \epsilon\geq \frac{\sum_{j=1}^{d_0}\widehat{\lambda}_j^{(r)} - \sum_{j=1}^{d_0}\widehat{\lambda}_{d+j}^{(r)}}{2r^2d_0}.
\end{align*}
where $\widehat{\lambda}_j^{(r)} = 0 $ for $j > p$.

\end{proof} 

\subsubsection{Proof of Theorem~\ref{thm:perturbation}}
\begin{proof}
    We adopted two inequality in \cite{yuUsefulVariantDavis2015a,caiRateoptimalPerturbationBounds2018a}. 
    Let $\Ic_{\epsilon} = \{i : \xv_i = \xv_{i,\epsilon}\}$ and $\Xv_{\epsilon}^{(r)\prime} = [\xv_{1,\epsilon}^{(r)},\dots,\xv_{n,\epsilon}^{(r)}]$. By Lemma~\ref{lem:kahan}, we have
    \begin{align*}
        \sin \widehat{\Theta}_{\epsilon}^{(r)} &\leq \frac{2 \|\frac{1}{n}\Xv^{(r)\prime}\Xv^{(r)\prime} - \frac{1}{n} \Xv_{\epsilon}^{(r)\prime}\Xv_{\epsilon}^{(r)\prime}\|}{\widehat{\lambda}_d^{(r)} - \widehat{\lambda}_{d+1}^{(r)}}\\
        &= \frac{2\|\frac{1}{n} (\sum_{i \not\in I_{\epsilon}}\xv_i^{(r)}\xv_i^{(r)\prime} - \xv_{i,\epsilon}^{(r)}\xv_{i,\epsilon}^{(r)\prime})\|}{\widehat{\lambda}_d^{(r)} - \widehat{\lambda}_{d+1}^{(r)}}\\
        &\leq \frac{2r^2\epsilon}{\widehat{\lambda}_d^{(r)} - \widehat{\lambda}_{d+1}^{(r)}}.
    \end{align*}
    
    For the second upper bound, $\sin \widehat{\Theta}_{\epsilon}^{(r)} \leq \frac{r^2 \epsilon}{\widehat{\lambda}_d^{(r)}- \widehat{\lambda}_{d+1}^{(r)} - 2r^2 \epsilon}$, assume that $\widehat{\lambda}_d^{(r)}- \widehat{\lambda}_{d+1}^{(r)} > 4r^2\epsilon$. Let $\widehat{\Wv} = [\widehat{\Wv}_d,\widehat{\Wv}_\perp]$ be the orthogonal matrix where $\widehat{\Wv}_d$ is the right singular vector corresponding to the $d$ largest singular values of $\Xv^{(r)}$. Using the proposition in \cite{caiRateoptimalPerturbationBounds2018a}, we have
\begin{align*}
    \sin\widehat{\Theta}_{\epsilon}^{(r)} 
    &\leq \dfrac{\sigma_d(\Xv_{\epsilon}^{(r)}\widehat{\Wv}_d) \|\Pi_{(\Xv_{\epsilon}^{(r)}{\widehat{\Wv}_d})}  \Xv_{\epsilon}^{(r)} \widehat{\Wv}_\perp\|}{\sigma_d^2(\Xv_{\epsilon}^{(r)}{\widehat{\Wv}_d}) - \sigma_{d+1}^2(\Xv_{\epsilon}^{(r)})}
    \end{align*}
    where $\sigma_j(\Av)$ is the $j$th largest singular value of $\Av$, $\Pi_{(\Xv_{\epsilon}^{(r)}{\widehat{\Wv}_d})}$ is the projection operator onto the column space of $\Xv_{\epsilon}^{(r)}{\widehat{\Wv}_d}$. The inequality holds when $\sigma_d^2(\Xv_{\epsilon}^{(r)}\widehat{\Wv}_d) - \sigma_{d+1}^2(\Xv_{\epsilon}^{(r)}) > 0$. Then,
    \begin{align*}
        &\sigma_d^2(\Xv_{\epsilon}^{(r)}{\widehat{\Wv}_d}) \\
        =& \sigma_{d}(\widehat{\Wv}_d'\Xv_{\epsilon}^{(r)\prime}\Xv_{\epsilon}^{(r)}\widehat{\Wv}_d)\\
        \geq& \sigma_{d}(\widehat{\Wv}_d'\sum_{i \in I_{\epsilon}}\xv_i^{(r)}\xv_i^{(r)\prime}\widehat{\Wv}_d)\\
        \geq& \sigma_{d}(\widehat{\Wv}_d'\sum_{i=1}^n\xv_i^{(r)}\xv_i^{(r)\prime}\widehat{\Wv}_d) - r^2m_n\\
        =& n \widehat{\lambda}_d^{(r)} - r^2m_n,\\
        &\sigma_{d+1}^2(\Xv_{\epsilon}^{(r)}) \\
        =& \sigma_{d+1}(\Xv_{\epsilon}^{(r)\prime}\Xv_{\epsilon}^{(r)})\\
        \leq& \sigma_{d+1}(\sum_{i \in I_{\epsilon}}\xv_i^{(r)}\xv_i^{(r)\prime}) + r^2m_n\\
        \leq& \sigma_{d+1}(\sum_{i=1}^n\xv_i^{(r)}\xv_i^{(r)\prime}) + r^2m_n\\
        =& n \widehat{\lambda}_{d+1}^{(r)} + r^2 m_n,\\
        &\|\Pi_{(\Xv^{(r)}_{\epsilon}\widehat{\Wv}_d)}  \Xv^{(r)}_{\epsilon} \widehat{\Wv}_\perp\|^2\\
        =& \sigma_1(\widehat{\Wv}_\perp'\Xv^{(r)\prime}_{\epsilon} \Pi_{(\Xv^{(r)}_{\epsilon}\widehat{\Wv}_d)} \Xv^{(r)}_{\epsilon}\widehat{\Wv}_\perp )\\   =&\sigma_1(\widehat{\Wv}_\perp'\Xv^{(r)\prime}_{\epsilon}\Xv^{(r)}_{\epsilon}\widehat{\Wv}_d\times (\widehat{\Wv}_d'\Xv^{(r)\prime}_{\epsilon}\Xv^{(r)}_{\epsilon}\widehat{\Wv}_d)^{-1}\times \widehat{\Wv}_d'\Xv^{(r)\prime}_{\epsilon}\Xv^{(r)}_{\epsilon}\widehat{\Wv}_\perp )\\
        \leq& \frac{\sigma_1^2(\widehat{\Wv}_\perp'\Xv^{(r)\prime}_{\epsilon} \Xv^{(r)}_{\epsilon}\widehat{\Wv}_d)}{\sigma_d(\widehat{\Wv}_d'\Xv^{(r)\prime}_{\epsilon}\Xv^{(r)}_{\epsilon}\widehat{\Wv}_d)},\\
        &\sigma_1(\widehat{\Wv}_\perp'\Xv^{(r)\prime}_{\epsilon} \Xv^{(r)}_{\epsilon}\widehat{\Wv}_d)\\
        =& \sigma_1(\widehat{\Wv}_\perp'\Xv^{(r)\prime}_n \Xv^{(r)}_n\widehat{\Wv}_d - \widehat{\Wv}_\perp'\sum_{i \not \in I_{\epsilon}}(\xv^{(r)}_i\xv^{(r)\prime}_i - \xv^{(r)}_{i,\epsilon}\xv^{(r)\prime}_{i,\epsilon})\widehat{\Wv}_d)\\
    \leq& \sigma_1(\widehat{\Wv}_\perp'\sum_{i \not \in I_{\epsilon}}(\xv^{(r)}_i\xv^{(r)\prime}_i - \xv^{(r)}_{i,\epsilon}\xv_{i,\epsilon}^{(r)\prime})\widehat{\Wv}_d)\\
        \leq &r^2 m_n.
    \end{align*}
    Thus we have
    \begin{align*}
        \sin\widehat{\Theta}_{\epsilon}^{(r)} \leq \dfrac{r^2 \epsilon }{\widehat{\lambda}_d^{(r)} - \widehat{\lambda}_{d+1}^{(r)} - 2r^2\epsilon}.
    \end{align*}
\end{proof}

\subsection{Empirical Expectation example generation in Section~\ref{sec:effectR}}\label{sec:singeneration}
Let $n = 200$, $p = 100$, $d = 1$, and $\Sigmav = \diag(100, 1, \dots, 1)$. The data points $\xv_1, \dots, \xv_n$ are independently generated from the multivariate Gaussian distribution $N_p(\0v, \Sigmav)$ or $t_3(\0v, \Sigmav)$ with $\0v = (0, \dots, 0)' \in \Rb^p$. Here $\Sigmav$ in $t_3(\0v, \Sigmav)$ implies the covariance matrix of $t_3(\0v, \Sigmav)$.

For the outliers, we replace the first $m := 0.05n = 10$ data points $\xv_1, \dots, \xv_m$ with $\zv_i = (0, 100np, 0, \dots, 0)$ for $i = 1, \dots, 0.05n$. Consequently, the contaminated data with $m$ outliers is denoted by $$\Xv_{m/n} = [\zv_1, \dots, \zv_m,  \xv_{m+1}, \dots, \xv_n]'.$$ We replicate the experiment 100 times.

\subsection{Empirical Expectation example generation in Section~\ref{sec:numeric}}\label{sec:numeric_generation}
For $k=1,\dots,4$, we set the dimension $p_k = 1000k$, and the sample size $n_k$ to satisfy $p_k/n_k = 1/2k$. For the non-spiked model where the eigenvalues of the covariance matrix $\Sigmav$ are constant, $\Sigmav = \diag(2^2,3^2)$. For the spiked model, we set $\Sigmav = \diag(2^2\sqrt{p_k}, 3^2\sqrt{p_k})$. The radii we concern are $r_{1k} = 1$, $r_{2k} = \sqrt{p_k}$, and $r_{3k} = \sqrt{p_k \log p_k}$. The data points $\xv_1, \dots, \xv_n$ are independently generated from the multivariate Gaussian distribution $N_p(\0v, \Sigmav)$ or $t_3(\0v, \Sigmav)$. Here $\Sigmav$ in $t_3(\0v, \Sigmav)$ implies the covariance matrix of $t_3(\0v, \Sigmav)$.

For the outliers, we replace the first $m := 2$ data points $\xv_1, \xv_2$ with $\zv_i = (\0v_d', n_k p_k, \0v_{p_k -d -1}')$ for $i = 1, 2$. Here $\0v_l$ is the zero vector with the length $l$. Consequently, the contaminated data with $m$ outliers is denoted by $\Xv_{m/n} = [\zv_1, \zv_2, \xv_{3}, \dots, \xv_n]'$. We replicate the experiment 10 times.

\subsection{Data Generation for Estimated Lower Bounds in Section~\ref{sec:breakdown_winsor}}\label{sec:breakdown_generation}
Let $n=1000$, $p=4$, $d=2$, and $\Sigmav = \diag(5^2,5^2,5,1)$. The data points $\xv_1,\dots,\xv_n$ are independently generated from the multivariate Gaussian distribution $N_p(\0v , \Sigmav)$, where $\0v = (0,0)'$. The data matrix is represented as $\Xv = [\xv_1,\dots,\xv_n]'\in \Rb^{n\times p}$. We replicate the experiment 1000 times.

\subsection{Toy example generation in Section~\ref{sec:pertur}}\label{sec:toygeneration}
Let $n=1000$, $p=2$, $d=1$, and $\Sigmav = \diag(5^2,1)$. The data points $\xv_1,\dots,\xv_n$ are independently generated from the bivariate Gaussian distribution $N_p(\0v , \Sigmav)$, where $\0v = (0,0)'$. The data matrix is represented as $\Xv = [\xv_1,\dots,\xv_n]'\in \Rb^{n\times p}$. Define $r = \mathrm{med}_i{\|\xv_i\|_2}$. To introduce outliers, we replace the first $m$ data points $\xv_1,\dots,\xv_m$ with $\zv_i= (0,\max_i{\|\xv_i\|_2^2} + 100)$ for $i=1,\dots,m$. Note that the outliers need not be excessively large, as we are only concerned with the median radius. Consequently, the contaminated data with $m$ outliers is denoted by $\Xv_{m/n} = [\zv_1,\dots,\zv_m,\xv_{m+1},\dots,\xv_n]'$.

\end{document}